%% file: main.tex
\pdfoutput=1
\documentclass{article}
\pdfoutput=1
 
\PassOptionsToPackage{round}{natbib}
\usepackage[final,nonatbib]{neurips_2020}

\usepackage[utf8]{inputenc} % allow utf-8 input
\usepackage[T1]{fontenc}    % use 8-bit T1 fonts
\usepackage[hidelinks]{hyperref}
\usepackage{url}            % simple URL typesetting
\usepackage{booktabs}       % professional-quality tables
\usepackage{amsfonts}       % blackboard math symbols
\usepackage{nicefrac}       % compact symbols for 1/2, etc.
\usepackage{microtype,soul}      % microtypography
\usepackage{natbib}
\input{weichao}
\title{\texttt{POLY-HOOT}: Monte-Carlo Planning in Continuous Space MDPs with Non-Asymptotic Analysis}

\author{%
Weichao Mao\\
  ECE and CSL\\
  University of Illinois at Urbana-Champaign\\
  \texttt{weichao2@illinois.edu}
  \And
Kaiqing Zhang\\
  ECE and CSL\\
  University of Illinois at Urbana-Champaign\\
  \texttt{kzhang66@illinois.edu}
  \AND 
\hspace{0.9cm}Qiaomin Xie\\
\hspace{0.8cm}  ORIE\\
\hspace{0.85cm}  Cornell University\\
\hspace{0.95cm}  \texttt{qiaomin.xie@cornell.edu}
  \And
\hspace{0.83cm}Tamer Ba\c{s}ar\\
  \hspace{0.88cm}ECE and CSL\\
  \hspace{0.9cm}University of Illinois at Urbana-Champaign\\
   \hspace{0.9cm}\texttt{basar1@illinois.edu}
}

\begin{document}

\maketitle
\begin{abstract}
Monte-Carlo planning, as exemplified by Monte-Carlo Tree Search (MCTS), has demonstrated remarkable performance in applications with finite spaces. In this paper, we consider Monte-Carlo planning in an environment with continuous state-action spaces, a much less understood problem with important applications in control and robotics. We introduce \texttt{POLY-HOOT}, an algorithm that augments MCTS with a continuous armed bandit strategy named Hierarchical Optimistic Optimization (HOO) \citep{bubeck2011x}. Specifically, we enhance HOO by using an appropriate \emph{polynomial}, rather than \emph{logarithmic}, bonus term in the upper confidence bounds.  Such a polynomial bonus is motivated by its empirical successes in AlphaGo Zero~\citep{silver2017mastering}, as well as its significant role in achieving theoretical guarantees of finite space MCTS~\citep{shah2019reinforcement}. We investigate, for the first time, the regret of the enhanced HOO algorithm in non-stationary bandit problems. Using this result as a building block, we establish non-asymptotic convergence guarantees for \texttt{POLY-HOOT}: the value estimate converges to an arbitrarily small neighborhood of the optimal value function at a polynomial rate. We further provide experimental results that corroborate our theoretical findings. 
\end{abstract}

\input{1introduction}

\input{2preliminaries}

\input{3algorithm}

\input{4analysis}

\input{5simulations}

\section{Conclusions}\label{sec:conclusions}
In this paper, we have considered Monte-Carlo planning in an environment with continuous state-action spaces. We have introduced \texttt{POLY-HOOT}, an algorithm that augments MCTS with a continuous armed bandit strategy HOO. We have enhanced HOO with an appropriate polynomial bonus term in the upper confidence bounds, and investigated the regret of the enhanced HOO algorithm in non-stationary bandit problems. Based on this result, we have established non-asymptotic convergence guarantees for \texttt{POLY-HOOT}. Experimental results have further corroborated our theoretical findings. Our theoretical results have advocated the use of non-stationary bandits with polynomial bonus terms in MCTS, which might guide the design of new planning algorithms in continuous spaces, with potential applications in robotics and control, that enjoy better empirical performance as well.

\section*{Broader Impact}
We believe that researchers of planning, reinforcement learning, and multi-armed bandits, especially those who are interested in the theoretical foundations, would benefit from this work. In particular, prior to this work, though intuitive, easy-to-implement, and empirically widely-used, a theoretical analysis of Monte-Carlo tree search (MCTS) in continuous domains had not been established through the lens of non-stationary bandits. In this work, inspired by the recent advances in finite-space Monte-Carlo tree search, we have provided such a result, and thus theoretically justified the efficiency of MCTS in continuous domains. 
  
Although Monte-Carlo tree search has demonstrated great performance in a wide range of applications, theoretical explanation of its empirical successes is relatively lacking. Our theoretical results have advocated the use of non-stationary bandit algorithms, which might guide the design of new planning algorithms that enjoy better empirical performance in practice. Our results might also be helpful for researchers interested in robotics and control applications, as our algorithm can be readily applied to such planning problems with continuous domains. 

As a theory-oriented work, we do not believe that our research will cause any ethical issue, or put anyone at any disadvantage.

\begin{ack}
	We thank Bin Hu for helpful comments on an earlier version of the paper. Research of the three authors from Illinois was supported in part by Office of Naval Research (ONR) MURI Grant N00014-16-1-2710, and in part by the US Army Research Laboratory (ARL) Cooperative Agreement W911NF-17-2-0196. Q.~Xie is
	partially supported by NSF grant 1955997.
\end{ack}

\bibliographystyle{abbrvnat}
\bibliography{ref}

\input{6supplementary}

\end{document}

%% file: weichao.tex
\usepackage{amsfonts}
\usepackage{xcolor}
\let\proof\relax
\let\endproof\relax
\usepackage{amsthm}
\usepackage{amsmath}
\usepackage{amssymb}
\usepackage{soul}
\usepackage{comment}
\usepackage{enumitem}
\usepackage{bbm}
\usepackage{graphicx}
\usepackage{float}
\usepackage{subfigure}
\usepackage[ruled,linesnumbered,noend,noline]{algorithm2e}

\newtheorem{definition}{Definition}
\newtheorem{assumption}{Assumption}
\newtheorem{theorem}{Theorem}
\newtheorem{lemma}{Lemma}
\newtheorem{proposition}{Proposition}
\newenvironment{proofsketch}{\proof}{\endproof}
\theoremstyle{remark}
\newtheorem{remark}{Remark}

\newcommand{\ra}{\rightarrow}
\newcommand{\mb}[1]{\mathbb{#1}}

\newcommand{\mc}[1]{\mathcal{#1}}

\newcommand{\rr}{\mathbb{R}}
\newcommand{\ee}{\mathbb{E}}
\newcommand{\pp}{\mathbb{P}}

\renewcommand{\epsilon}{\varepsilon}

\newcommand{\abs}[1]{\left| #1 \right|}
\newcommand{\norm}[1]{\left\| #1 \right\|}
\newcommand{\prob}[1]{\pp\left( #1 \right)}	
\newcommand{\e}[1]{\ee\left[#1\right]}

\newcommand{\ceil}[1]{\left\lceil #1 \right\rceil}

\newcommand{\indicator}{\mathbbm{1}}

\definecolor{softpink}{RGB}{216, 52, 108}

%% file: 1introduction.tex
\section{Introduction}
Monte-Carlo tree search (MCTS) has recently demonstrated remarkable success in deterministic games, especially in the game of Go~\citep{silver2017mastering}, Chess and Shogi~\citep{silver2017masteringchess}. It is also among the very few viable approaches to problems with partial observability, e.g., Poker~\citep{rubin2011computer}, and problems involving highly complicated strategies like real-time strategy games~\citep{uriarte2014game}. However, most Monte-Carlo planning solutions only work well in finite state and action spaces, and are generally not compatible with continuous action spaces with enormous branching factors. Many important applications such as robotics and control require planning in a continuous state-action space, for which feasible solutions,  especially those with theoretical guarantees, are scarce. In this paper, we aim to develop an MCTS method for \emph{continuous} domains with \emph{non-asymptotic  convergence} guarantees.

Rigorous analysis of MCTS is highly non-trivial even in finite spaces. One crucial difficulty stems from the fact that the state-action value estimates in MCTS are non-stationary over multiple simulations, because the policies in the lower levels of the search tree are constantly changing. Due to the strong non-stationarity and interdependency of rewards, the reward concentration hypothesis made in the seminal work of~\cite{kocsis2006bandit}---which provides one of the first theoretical analysis of bandit-based MCTS---turns out to be unrealistic. Hence, the convergence analysis given in~\cite{kocsis2006bandit} is unlikely to hold in general.
Recently a rigorous convergence result is established in~\cite{shah2019reinforcement}, based on further investigation of \emph{non-stationary multi-armed bandits} (MABs).

Besides the non-stationarity issue inherent in MCTS analysis, an additional challenge for  continuous domains lies in balancing the trade-off between generating fine-grained samples across the entire continuous action domain to ensure optimality, and guaranteeing sufficient exploitation of the sampled actions for accurate estimations. To tackle this challenge, a natural idea is to manually discretize the action space and then solve the resulting discrete problem using a discrete-space planning algorithm. However, this approach inevitably requires a hyper-parameter pre-specifying the level of discretization, which in turn leads to a fundamental trade-off between the computational complexity and the optimality of the planning solution: coarse discretization often fails to identify the optimal continuous action, yet fine-grained discretization leads to a large action space and heavy computation. 

In this paper, we consider Monte-Carlo planning in continuous space Markov Decision Processes (MDPs) without manually discretizing the action space. Our algorithm integrates MCTS with a continuous-armed bandit strategy, namely Hierarchical Optimistic Optimization (HOO)~\citep{bubeck2011x}. Our algorithm \emph{adaptively partitions}  the action space and quickly identifies the region of potentially optimal actions in the continuous space, which alleviates the inherent difficulties encountered by pre-specified discretization. The integration of MCTS with HOO has been empirically evaluated in~\cite{mansley2011sample}, under the name of the Hierarchical Optimistic Optimization applied to Trees (HOOT) algorithm. HOOT directly replaces the UCB1 bandit algorithm~\citep{auer2002finite} used in finite-space MCTS with the HOO strategy. However, this algorithm has a similar issue as that in~\cite{kocsis2006bandit}, as they both use a \emph{logarithmic}  bonus term for bandit exploration instead of a \emph{polynomial} term. As pointed out in~\cite{shah2019reinforcement} and mentioned above, convergence guarantees of these algorithms are generally unclear due to the lack of concentration of non-stationary rewards. In this work, we enhance the HOO strategy with a polynomial bonus term to account for the non-stationarity. As we will show in our theoretical results, our algorithm, Polynomial Hierarchical Optimistic Optimization applied to Trees (\texttt{POLY-HOOT}), provably converges to an arbitrarily small neighborhood of the optimum at a polynomial rate.

\noindent\textbf{Contributions.} First, we enhance the continuous-armed bandit strategy HOO, and analyze its regret concentration rate in a non-stationary setting, which may also be of independent theoretical interest in the context of bandit problems. Second, we build on the enhanced HOO to design a Monte-Carlo planning algorithm \texttt{POLY-HOOT} for solving continuous space MDPs. Third, we generalize the recent analytical framework developed for finite-space MCTS~\citep{shah2019reinforcement} and prove that the value estimate of \texttt{POLY-HOOT} converges to an arbitrarily small neighborhood of the optimal value function at a polynomial rate. We note that HOOT is among the very few MCTS algorithms for continuous spaces and popular in practice. \texttt{POLY-HOOT}  improves upon HOOT and provides theoretical justifications thereof. Finally, we present experimental results which corroborate our theoretical findings and demonstrate the superior performance of \texttt{POLY-HOOT}.

\noindent\textbf{Related Work.} 
One of the most popular MCTS methods is the Upper Confidence Bounds applied to Trees (UCT) algorithm~\citep{kocsis2006bandit}, which applies the  UCB1~\citep{auer2002finite} bandit algorithm for action selection. A convergence result of UCT is provided in~\cite{kocsis2006bandit}. However, this result relies on the assumption that bandit regrets under UCB1 concentrate exponentially, which is unlikely to hold in general. Recent work in~\cite{shah2019reinforcement} provides a complete analysis of UCT through a further study of non-stationary bandit algorithms using polynomial bonus. Our analysis falls into the general framework proposed therein. We note that many variations and enhancements of MCTS have been developed~\citep{coquelin2007bandit,schadd2008single,kaufmann2017monte,xiao2019maximum,jonsson2020}; we refer interested readers to a survey by~\cite{browne2012survey}. We remark that most variants are restricted to finite-action problems.

MCTS for continuous-space MDPs has been relatively less studied. In the literature a progressive widening (PW) technique~\citep{chaslot2007progressive,auger2013continuous} is often used to discretize the action space and ensure sufficient exploitation. However, PW mainly concerns \emph{when} to sample a new action, but not \emph{how}. For example, \cite{auger2013continuous} draws an action uniformly at random, which is sample-inefficient compared to our bandit-based action selection.  
Popular in empirical work is the HOOT algorithm in~\citep{mansley2011sample}, which directly replaces the UCB1 bandit strategy in UCT with HOO. This work does not provide theoretical guarantees, and given the non-stationarity of the bandit rewards, there is a good reason to believe that a more sophisticated variant of HOO is needed. An open-loop planning solution named Hierarchical Open-Loop Optimistic Planning (HOLOP) is proposed and empirically evaluated in~\cite{weinstein2012bandit}. In~\cite{yee2016monte}, MCTS is combined with kernel regression, and the resulting algorithm demonstrates good empirical performance.  More recently, \cite{kimmonte} proposes to partition the continuous space based on the Voronoi graph, but they focus on deterministic rewards and do not utilize bandits to \emph{guide the exploration and exploitation} of actions, which is the main focus of our work.

\noindent\textbf{Outline.}
The rest of the paper is organized as follows: In Section~\ref{sec:preliminaries}, we introduce the mathematical formulation and some preliminaries. In Section~\ref{sec:algorithm}, we present our \texttt{POLY-HOOT} algorithm. In Section~\ref{sec:analysis}, we provide our analysis of the non-stationary bandits and our main results on the convergence of \texttt{POLY-HOOT}. Simulation results are provided in Section~\ref{sec:simulations}. Finally, we conclude our paper in Section~\ref{sec:conclusions}. The detailed algorithms and proofs of the theorems can be found in the appendix.

%% file: 2preliminaries.tex
\section{Preliminaries}\label{sec:preliminaries}

\subsection{Markov Decision Processes}

We consider an infinite-horizon discounted MDP defined by a 5-tuple $(S, A, T, R, \gamma)$,  where $S\subseteq \rr^n$ is the continuous state space, $A\subseteq \rr^m$ the continuous action space, $T:S\times A \ra S$ the deterministic transition function, $R:S\times A\ra [-R_{max}, R_{max}]$ the (bounded) stochastic reward function, and $\gamma \in (0,1)$ is the discount factor. We do not require $S$ and $A$ to be compact, thus our theory covers many control applications with possibly unbounded state-action spaces. 
The assumption of deterministic state transitions is common in the MCTS literature~\citep{browne2012survey,shah2019reinforcement,kimmonte}, as MCTS was historically introduced and popularly utilized in problems like Go~\citep{gelly2006modification,silver2017mastering} and Atari games~\citep{guo2014deep}. For simplicity we use the notation $ s\circ a \triangleq T(s,a) $ to denote the next state deterministically reached by taking action $a\in A$ at the current state $s\in S$. 

A policy $\pi:S\ra A$ specifies the action $ a=\pi(s) $ taken at state $ s $. The value function $V^\pi:S\ra \rr$ of a policy $\pi$ is defined as the expected discounted sum of rewards following $\pi$ starting from the current state $s\in S$, i.e., $V^{\pi}(s)=\mathbb{E}_{\pi}\left[\sum_{t=0}^{\infty} \gamma^{t} R\left(s_{t}, a_{t}\right) | s_{0}=s\right]$. Similarly, define the state-action value function $Q^{\pi}(s,a)=\mathbb{E}_{\pi}\left[\sum_{t=0}^{\infty} \gamma^{t} R\left(s_{t}, a_{t}\right) | s_{0}=s, a_0 = a\right]$. The planner aims to find an optimal policy $\pi^*$ that achieves the maximum value $ V^{\pi^*}(s) = V^*(s) \triangleq \sup_\pi V^\pi(s)$ for all $s\in S$. 

We consider the problem of computing the optimal value function for any given input state, with access to a generative model (or simulator) of the MDP. A generative model provides a randomly sampled next state and reward, when given any state-action pair $(s,a)$ as input. Our algorithms and results readily extend to learning the optimal policy or Q-function.

\subsection{Monte-Carlo Tree Search}

To estimate the optimal value of a given state, Monte-Carlo tree search (MCTS) builds a multi-step look-ahead tree, with the state of interest as the root node, using Monte-Carlo simulations~\citep{browne2012survey}. Each node in the tree represents a state, and each edge represents a state-action pair that leads to a child node denoting the subsequent state. 
At each iteration, starting from the root node, the algorithm selects actions according to a \emph{tree policy} and obtains samples from the generative model until reaching a leaf node. An estimate for the value of leaf node can be either obtained by simulations of a roll-out policy or given by some function approximation. The leaf node estimate and samples generated along the path are then backed-up to update the statistics of selected nodes. The tree policy plays a key role of balancing exploration-exploitation. The most popular tree policy is UCT~\citep{kocsis2006bandit}, which selects children (actions) according to the Upper Confidence Bound (UCB1)~\citep{auer2002finite} bandit algorithm. Note that UCT, and most variants thereof, are restricted to the finite action setting. 

A major challenge in the theoretical analysis of any MCTS algorithm is the \emph{non-stationarity} of bandit rewards. Specifically, since the policies at the lower level bandits of MCTS are constantly changing, the reward sequences for each bandit agent drift over time, causing the reward distribution to be highly non-stationary. The performance of each bandit depends on the results of a chain of bandits at the lower levels, and this hierarchical inter-dependence of bandits makes the analysis highly non-trivial.
A complete solution to address this non-stationarity has been given recently in~\cite{shah2019reinforcement}, where the authors inductively show the polynomial concentration of rewards by leveraging a non-stationary bandit algorithm with a \emph{polynomial} bonus term. Our approach in the continuous case is based upon a similar reasoning as in~\cite{shah2019reinforcement}.

\subsection{Hierarchical Optimistic Optimization}

HOO~\citep{bubeck2011x} is an extension of finite-armed bandit algorithms to problems with arms living in an arbitrary measurable space, e.g., the Euclidean space. HOO incrementally builds a binary tree covering of the continuous action space $X$. Each node in the tree covers a subset of $X$. This subset is further divided into two, corresponding to the two child nodes. HOO selects an action by following a path from the root node to a leaf node, and at each node it picks the child node that has the larger upper confidence bound (to be precise, larger $B$-value; see equation~\eqref{eqn:Bvalue}) for the reward. In this manner, HOO adaptively subdivides the action space and quickly focuses on the area where potentially optimal actions lie in. 

Following the notations in~\cite{bubeck2011x}, we index the nodes in the above HOO tree  
by pairs of integers $(h,i)$,\footnote{We use $h$ and $H$ to index the depth in the HOO tree, and use $d$ and $D$ to index the depth in the MCTS tree.} where $h\geq 0$ denotes the depth of the node, and $1\leq i\leq 2^h$ denotes its index on depth $h$. In particular, the root node is $(0,1)$; the two children of $(h,i)$ are $(h+1,2i-1)$ and $(h+1,2i)$. Let $\mc{P}_{h,i}\subseteq X$ be the domain covered by the node $(h,i)$. By definition, we have $\mc{P}_{0,1} = X$ and $
\mathcal{P}_{h, i}=\mathcal{P}_{h+1,2 i-1} \cup \mathcal{P}_{h+1, 2 i}
,\forall  h \geq 0 \text { and } 1 \leq  i \leq 2^{h}$. Let $\mc{C}(h,i)$ denote the set of all descendants of node $(h,i)$. Let $(H_t,I_t)$ denote the node played by HOO at round $t$, with observed reward $ Y_t $. Then the number of times that a descendant of $(h,i)$ has been played up to and including round $n$ is denoted by $T_{h, i}(n)=\sum_{t=1}^{n} \indicator_{\left\{\left(H_{t}, I_{t}\right) \in \mathcal{C}(h, i)\right\}},$ and the empirical average of rewards is defined as $\widehat{\mu}_{h, i}(n)=\frac{1}{T_{h, i}(n)} \sum_{t=1}^{n} Y_{t} \indicator_{\left\{\left(H_{t}, I_{t}\right) \in \mathcal{C}(h, i)\right\}}.$ 

In the original HOO algorithm of~\citet{bubeck2011x}, the upper confidence bound of a node $(h,i)$ is constructed using a logarithmic bonus term: 
\begin{equation}\label{eqn:upperconfidence}
U_{h, i}(n)=\left\{\begin{array}{ll}\widehat{\mu}_{h, i}(n)+ \sqrt{\frac{2\ln n}{T_{h,i}(n)}} +\nu_{1} \rho^{h}, & \text { if } T_{h, i}(n)>0, \\ \infty, & \text { otherwise ,}\end{array}\right.
\end{equation}% 
where $\nu_1$ and $\rho$ are two constants that characterize the reward function and the action domain.
Given $ U_{h,i}(n), $ one further introduces a critical quantity termed the $B$-values:
\begin{equation}\label{eqn:Bvalue} 
B_{h, i}(n)=\left\{\begin{array}{ll}\min \left\{U_{h, i}(n), \max \left\{B_{h+1,2 i-1}(n), B_{h+1,2 i}(n)\right\}\right\}, & \text { if }(h, i) \in \mathcal{T}_{n}, \\ \infty, & \text { otherwise, }\end{array}\right.
\end{equation} 
where $\mc{T}_n$ is the set of nodes that are already included in the binary tree at round $n$. Starting from the root node, HOO iteratively selects a child node with a larger $B$-value until it reaches a leaf node, which corresponds to an arm of the bandit to be pulled.

%% file: 3algorithm.tex
\section{Algorithm: \texttt{POLY-HOOT} }\label{sec:algorithm}

Our algorithm for continuous space MCTS, Polynomial Hierarchical Optimistic Optimization applied to Trees (\texttt{POLY-HOOT}), is presented in Algorithm~\ref{alg:POLY-HOOT}.

\begin{algorithm}[!htbp]
	%\SetAlgoLined	
	\textbf{Input:} value oracle at leaf nodes $\hat{V}$, root node $s^{(0)}$, maximum search depth $D$, number of MCTS simulations $n$, and parameters $\{\alpha^{(i)}\}_{i=0}^{D-1},\{\xi^{(i)}\}_{i=0}^{D-1},\{\eta^{(i)}\}_{i=0}^{D-1}$.
	
	\textbf{Output:} value estimate of the root node $s^{(0)}$.\\
	\For{simulation round $t \gets 1$ to $n$}
	{
		\For{depth $d \gets 0 $ to $D-1$}
		{
			$a^{(d)}\gets$ $HOO\_query(d, s^{(d)}, t)$ with depth limitation $\bar{H}$\;
			$r^{(d)} \sim R(s^{(d)}, a^{(d)})$\;
			$s^{(d+1)}\gets s^{(d)}\circ a^{(d)}$\;
		}
		$r^{(D)}(s^{(D)}) \gets \hat{V}(s^{(D)})$\;
		%$\tilde{v}^{(D)}(s^{(D)}) \gets \tilde{v}^{(D)}(s^{(D)})+\hat{V}(s^{(D)})$\;
		\For{depth $d \gets 0 $ to $D-1$}
		{
			$Y^{(d)} \gets  r^{(d)}+\gamma r^{(d+1)}+\cdots+\gamma^{D-d-1} r^{(D-1)}+\gamma^{D-d} r^{(D)}(s^{(D)})$\;
			$\tilde{v}^{(d)}(s^{(d)})\gets \tilde{v}^{(d)}(s^{(d)})+Y^{(d)}$\;
			$HOO\_update(d, s^{(d)},t, Y^{(d)})$ using parameters $\alpha^{(d)},\xi^{(d)}$ and $\eta^{(d)}$\;
		}
	}
	\Return $\tilde{v}^{(0)}(s^{(0)})/n$.
	\caption{\texttt{POLY-HOOT}}\label{alg:POLY-HOOT}
\end{algorithm}

\texttt{POLY-HOOT} follows a similar framework as the classic UCT algorithm, but has the following critical enhancements to handle continuous spaces with provable convergence guarantees.

\textbf{1. HOO-Based Action Selection.} We replace the discrete UCB1 bandit agent with a continuous-armed HOO agent. In this case, each node in the Monte-Carlo tree is itself a HOO tree. In particular, \texttt{POLY-HOOT} invokes the HOO algorithm through two functions: the $HOO\_query$ function selects actions; after the action is taken and the reward is realized, the $HOO\_update$ function updates the reward information at each HOO agent along the Monte-Carlo sampling path. Detailed descriptions are provided in Appendix~\ref{appendix:algorithm}.

\textbf{2. Polynomial Bonus.} We replace the logarithmic bonus term used in the original HOO algorithm (Equation~\eqref{eqn:upperconfidence}) with a polynomial term. In particular, our algorithm constructs the upper confidence bound of a node $(h,i)$ as follows:
$$U_{h, i}(n)=\left\{\begin{array}{ll}\widehat{\mu}_{h, i}(n)+ n^{\alpha^{(d)}/\xi^{(d)}} T_{h,i}(n)^{\eta^{(d)} -1} +\nu_{1} \rho^{h}, & \text { if } T_{h, i}(n)>0, \\ \infty, & \text { otherwise ,}\end{array}\right.$$
where $\alpha^{(d)}, \xi^{(d)}$ and $\eta^{(d)}$ are constants to be specified later for each depth $d$ in MCTS. 
As shall become clear in the analysis, this polynomial bonus is critical in establishing convergence of MCTS. In particular, MCTS involves a hierarchy of bandits with non-stationary rewards, for which logarithmic bonus is no longer appropriate and does not guarantee (even asymptotic) convergence. Interestingly, the empirically successful AlphaGo Zero also uses polynomial bonus~\citep{silver2017mastering}. As in the original HOO, our algorithm navigates down the HOO tree using the $B$-value defined in~\eqref{eqn:Bvalue}, except that we plug in the above polynomial upper confidence bound $ U_{h,i}(n). $

\textbf{3. Bounded-Depth HOO Tree.} We place an upper bound $\bar{H}$ on the maximum depth of the HOO tree. Every time we reach a node at the maximum depth, the algorithm repeats the action  taken previously at that node. As such, our enhanced HOO stops exploring new actions after trying sufficiently many actions. In the original HOO strategy, the tree is allowed to extend infinitely deep, so that the action space can be discretized into arbitrarily fine granularity. 
When the bandit rewards are non-stationary, as in MCTS, this strategy might overlook the long-term optimal action and get stuck in a suboptimal area in the early stage of the tree search. 
On the contrary, our bounded depth HOO tree ensures that the actions already explored will be fully exploited against the non-stationarity of rewards. Our analysis shows that as long as the total number of actions tried is sufficiently large (i.e., $\bar{H}$ is chosen large enough), our algorithm still converges to an arbitrarily small neighborhood of the optimal value.

\subsection{Analysis Setup}\label{sec:setup}

Setting the stage for our theoretical analysis, we introduce several useful notations. For each HOO agent, let $X\subseteq A \subseteq [0,1]^m$ denote the continuous set of actions (i.e., arms) available at the current state. Each arm $x\in X$ is associated with a stochastic payoff distribution, which corresponds to the ``cost-to-go'' or $Q$-value of taking action $x$ at the current state of the MDP. The expectation of this reward function at time $ t $ is denoted by $f_t(x): X\ra \rr$, which is also termed the temporary mean-payoff function at time $t$. Note that in MCTS the temporary mean-payoff functions are non-stationary over time because the cost-to-go of an action depends on the actions to be chosen later in the lower levels of MCTS. Let $f$ be the limit of $f_t$ in the sense that $f_t$ converges to $f$ in $L^\infty$ at a polynomial rate:
$\norm{f_t-f}_\infty \leq  \frac{C}{t^\zeta},\ \forall t\geq 1$ for some constant $C>0$ and $\zeta \in (0,\frac{1}{2})$. The precise definition of $f_t$ and $f$, as well as the convergence of $f_t$, is formally established  in Theorem~\ref{thm:bandit}. We call $f$ the limiting mean-payoff function (or simply the mean-payoff function).

Since the rewards of the MDP are bounded by $R_{max}$, the bandit payoff for each node at depth $d$ is bounded by ${R_{max}}/{(1-\gamma)}$, and so is the limiting mean-payoff $f$ function.
Let $f^* = \sup_{x\in X} f(x)$ denote the optimal payoff at an HOO agent, and the random variable $X_t$ denote the arm selected by the agent at round $t$. The agent aims to minimize the regret in the first $n$ rounds: $R_n \triangleq nf^* - \sum_{t=1}^nY_t$, where $Y_t$ is the observed payoff of pulling arm $X_t$ at round $t$, with $\e{Y_t} = f_t(X_t)$.

We state two assumptions that  will be utilized throughout our analysis. These two assumptions are similar to Assumptions A1 and A2 in~\cite{bubeck2011x}. For each HOO agent in MCTS, given the parameters $\nu_1$ and $\rho\in (0,1)$, and the tree of coverings $(\mc{P}_{h,i})$, we assume that there exists a dissimilarity function $\ell: X\times X \to [0,\infty]$ such that the following holds.

\begin{assumption}\label{assumption:1}
	There exists a constant $\nu_2 >0$, such that for all integers $h\geq 0$,
	\begin{enumerate}[label=(\alph*)]
		\setlength{\itemsep}{0pt}
		\setlength{\parskip}{0pt}
		\vspace{-0.1in}
		\item diam$(\mc{P}_{h,i})\leq \nu_1 \rho^h, \forall 1\leq i \leq 2^h$, where diam$(A)\triangleq \sup_{x,y\in A}\ell(x,y)$;
		\item there exists an $x_{h,i}^\circ \in \mc{P}_{h,i}$, such that $\mathcal{B}_{h, i} \triangleq \mathcal{B}\left(x_{h, i}^{\circ}, \nu_{2} \rho^{h}\right) \subset \mc{P}_{h, i}, \forall 1 \leq i \leq 2^h$, where $\mc{B}(x,\varepsilon) \triangleq \{y\in X: \ell(x,y)<\varepsilon\}$ denotes an open ball centered at $x$ with radius $\varepsilon$;
		\item $\mathcal{B}_{h, i} \cap \mathcal{B}_{h, j}=\emptyset$ for all $1 \leq  i<j \leq  2^{h}$.
	\end{enumerate}
\end{assumption}
\begin{remark}
	Assumption~\ref{assumption:1} ensures that the diameter of $\mc{P}_{h,i}$ shrinks at a geometric rate as $h$ grows. {This is a mild assumption, which holds automatically in, e.g., compact Euclidean spaces. In particular, if the action space is a hyperrectangle, then} Assumption~\ref{assumption:1} is satisfied by setting the dissimilarity function $\ell$ to be some positive power of the Euclidean norm. For example, suppose that the action space is $[0,1]^2$. The tree covering can be generated by cutting the hyperrectangle of $\mc{P}_{h,i}$ at the midpoint of its longest side (ties broken arbitrarily) to obtain $\mc{P}_{h+1,2i-1}$ and $\mc{P}_{h+1,2i}$. Assumption~\ref{assumption:1} is satisfied with $\ell$ being the Euclidean norm and the parameters $\rho = \frac{1}{2}, \nu_1 = 8$, and $\nu_2 = \frac{1}{4}$. The general form of Assumption~\ref{assumption:1} allows more flexibility in the choice of $\ell$. 
\end{remark}

\begin{assumption}[Smoothness]\label{assumption:2}
	 The limiting mean-payoff function satisfies:
	\[
	f^{*}-f(y) \leq  f^{*}-f(x)+\max \left\{f^{*}-f(x), \ell(x, y)\right\}, \quad \forall x,y\in X.
	\]
\end{assumption}
\begin{remark}
Assumption~\ref{assumption:2} requires some smoothness of the mean-payoff function, and is milder than the common Lipschitz continuity  assumption $\abs{f(x)-f(y)}\leq \ell(x,y),\forall x,y\in X$.
In particular, it requires Lipschitz continuity only in the neighborhood of any global optimal arm $x^*$, and imposes a weaker constraint for other $x\in X$.  In the context of MDPs, this assumption stipulates that the $Q(s,a)$ function, after $d \in [1,D)$ steps of value iterations starting from $\hat{V}$, is a Lipschitz continuous function of the action $a$. Assumption~\ref{assumption:2} is satisfied by, e.g., Lipschitz MDPs~\citep{asadi2018lipschitz},\footnote{This is the class of MDPs whose reward functions and (possibly deterministic) state transitions satisfy certain smoothness criteria with respect to, say, the Wasserstein metric. As observed in~\cite{asadi2018lipschitz}, the Wasserstein metric is often more appropriate than  the Kullback-Leibler divergence metric in Lipschitz MDPs.} although this assumption holds much more generally. 
\end{remark}

%% file: 4analysis.tex
\section{Main Results}\label{sec:analysis}
In this section, we present our main results. Theorem~\ref{thm:mcts} establishes the non-asymptotic convergence rate of \texttt{POLY-HOOT}. Theorem~\ref{thm:bandit} characterizes the concentration rates of regret of enhanced HOO in a non-stationary bandit setting; this result serves as an important intermediate step in the analysis of \texttt{POLY-HOOT}. The proofs for Theorems~\ref{thm:mcts} and~\ref{thm:bandit} are given in Appendices~\ref{appendix:mcts} and~\ref{appendix:bandit}, respectively. 

\subsection{Convergence of \texttt{POLY-HOOT}}

\begin{theorem}\label{thm:mcts}
	Consider an MDP that satisfies Assumptions~\ref{assumption:1} and ~\ref{assumption:2}. For any $D\geq 1$, run $n$ rounds of MCTS simulations with parameters specified as follows:
	\begin{equation}\label{eqn:parameters}
	\begin{aligned}
	\alpha^{(d)} &= \left(1-\eta^{(d)}\right)\eta^{(d)}\xi^{(d)},\ &0 \leq d\leq D-1,\\
	\xi^{(d-1)} &= \left(\alpha^{(d)}-3\right)/2,\ &1 \leq d \leq D-1,\\
	\eta^{(d-1)} &= \frac{\frac{\alpha^{(d)}}{\xi^{(d)}(1-\eta^{(d)})}+d'+\frac{1}{1-\eta^{(d)}}}{1+d'+\frac{1}{1-\eta^{(d)}}},\ &1 \leq d \leq D-1,
	\end{aligned}
	\end{equation}
	where $d'>0$ is a constant to be specified in Definition~\ref{dfn:d'} (Appendix~\ref{appendix:bandit}). Suppose that $\xi^{(D-1)}>0$ and $\frac{1}{2} \leq \eta^{(D-1)} < 1$ are chosen large enough such that $\alpha^{(0)} > 3$, and  $\bar{H}$ satisfies $\rho^{\bar{H}} < n^{\eta^{(0)}-1}$. Then for each query state $s\in S$, the following result holds for the output $\hat{V}_n(s)$ of Algorithm~\ref{alg:POLY-HOOT}:
	\[
	\left|\mathbb{E}\left[\hat{V}_{n}(s)\right]-V^{*}(s)\right| \leq  O\left(\frac{1}{n^{\zeta}}\right) +  \gamma^D \varepsilon_0,
	\]
	where $\zeta \in (0,\frac{1}{2})$ satisfies $\zeta \leq 1-\eta^{(d)}, \forall\ 0\leq d\leq D-1$, and $\varepsilon_{0}=\big\|\hat{V}-V^{*}\big\|_{\infty}$ is the error in the value function oracle at the leaf nodes. 
\end{theorem}
\begin{proofsketch}
	MCTS can be viewed as a hierarchy of multi-armed bandits (in our case, continuous-armed bandits), one per each  node in the tree. In particular, the rewards of the bandit associated with each intermediate node are the rewards generated by the bandit algorithms for nodes downstream. Since the HOO policy is changing to balance exploitation-exploration, the resulting rewards are non-stationary. With this observation, the proof for Theorem~\ref{thm:mcts} can be broken down to the following three steps: 
	
	\textbf{1. Non-stationary bandits. } The first step concerns the analysis of a non-stationary bandit, which models the MAB at each node on the MCTS search tree. In particular, we show that if the rewards of a continuous-armed bandit problem satisfy certain convergence and concentration properties, then the regret induced by the enhanced HOO algorithm satisfies similar convergence and concentration guarantees. 
	The result is formally established in Theorem~\ref{thm:bandit}.

	\textbf{2. Induction step. } Since the rewards collected at one level of bandits constitute the bandit rewards of the level above it, we can apply the results of Step 1 recursively, from level $D-1$ upwards to the root node. We inductively show that the bandit rewards at each level $d$ of MCTS satisfy the properties required by Theorem~\ref{thm:bandit}, and hence we can propagate the convergence and concentration properties to  the bandit at level $d-1$, using the results of Theorem~\ref{thm:bandit}. The convergence result for the root node is  established by induction. 
	
	\textbf{3. Error from the oracle. } Finally, we consider the error induced by the leaf node estimator, i.e., the value function oracle $\hat{V}$. Given a value function oracle $\hat{V}$ for the leaf nodes, a depth-$D$ MCTS can be effectively viewed as $D$ steps of value iteration starting from $\hat{V}$~\citep{shah2019reinforcement}. Therefore, the error in the value function oracle $\hat{V}$ shrinks at a geometric rate of $\gamma$ due to the contraction mapping. 
\end{proofsketch}

Theorem~\ref{thm:mcts} implies that the value function estimate obtained by Algorithm~\ref{alg:POLY-HOOT} converges to the $\gamma^D\epsilon_0$-neighborhood of the optimal value function at a rate of $O(n^{-\zeta})$, where $\zeta \in (0,\frac{1}{2})$ depends on the parameters $\alpha^{(D-1)},\xi^{(D-1)}$, and $\eta^{(D-1)}$ we choose. Therefore, by setting the depth $D$ of MCTS appropriately, Algorithm~\ref{alg:POLY-HOOT} can output an estimate that is within an arbitrarily small neighborhood around the optimal values. 

\smallskip

\begin{remark}
	We remark on several technical challenges in the proof of Theorem~\ref{thm:mcts}. The first challenge is to transform a hierarchy of inter-dependent bandits into a recursive sequence of non-stationary bandit problems with unified form, which is highly non-trivial even in the finite case~\citep{shah2019reinforcement}. As far as we know, a general solution to non-stationary bandit problems with continuous domains is not available in the literature. Our enhanced HOO algorithm might be of independent research interest. Another challenge is to ensure sufficient exploitation in face of infinitely many candidate arms and strong non-stationarity of rewards. Existing solutions include uniformly sampling actions through progressive widening~\citep{auger2013continuous} and playing each action for a fixed amount of times~\citep{kimmonte}. Instead, our solution balances the trade-off between exploration and exploitation by using a limited depth HOO bandit, which makes our theoretical analysis highly non-trivial. 
\end{remark}

\subsection{Enhanced HOO in the Non-Stationary Setting}

The key step in the proof of Theorem~\ref{thm:mcts} is to establish the following result for the enhanced HOO bandit algorithm. Consider a continuous-armed bandit on the domain $X \subseteq [0,1]^m$, with non-stationary rewards bounded in $[-R, R]$ satisfying the following properties:

\quad A. Fixed-arm convergence: The mean-payoff function $f_n:X\ra \rr$ converges  to a function $f:X\ra \rr$ in $L^\infty$ at a polynomial rate:
\begin{equation}\label{eqn:convergence}
\norm{f_n-f}_\infty \leq  \frac{C}{n^\zeta},\ \forall n\geq 1,
\end{equation}
for some constant $C>0$ and $0 <\zeta < \frac{1}{2}$.

\quad B. Fixed-arm concentration: There exist constants $\beta > 1, \xi>0,$ and $1/2\leq \eta < 1$, such that for every $z\geq 1$ and every integer $n\geq 1$:
\begin{equation}\label{eqn:concentration}
\pp\left(\sum_{t=1}^n X_t - n f(x) \geq n^\eta z  \right) \leq \frac{\beta}{z^\xi}
\quad\text{and}\quad
\pp\left(\sum_{t=1}^n X_t - n f(x) \leq -n^\eta z \right) \leq \frac{\beta}{z^\xi},\ \forall x\in X,
\end{equation}
where $X_t$ denotes the random reward obtained by pulling arm $x\in X$ for the $t$-th time.

\begin{theorem}\label{thm:bandit}
	Consider a non-stationary continuous-armed bandit problem satisfying properties~\eqref{eqn:convergence} and~\eqref{eqn:concentration}. Suppose we apply the enhanced HOO agent defined in Algorithms~\ref{alg:HOO_query} and~\ref{alg:HOO_update} with parameters satisfying $\xi\eta(1-\eta)\leq \alpha <\xi(1-\eta)$, $\alpha>3$, and $\rho^{\bar{H}}< n^{\eta-1}$. 
	Let the random variable $Y_t$ denote the reward obtained at time $t$. Then the following holds:
	
	\quad A. Optimal-arm convergence: There exists some constant $C_0 > 0$, such that 
	\begin{equation}
	\abs{\frac{1}{n}\e{\sum_{t=1}^n Y_t}-f^*} \leq  \frac{C_0}{n^{\zeta}},
	\end{equation}
	where $0 < \zeta \leq \frac{1-\frac{\alpha}{\xi (1-\eta)}}{1+d'+\frac{1}{1-\eta}}$.
	
	\quad B. Optimal-arm concentration: There exist constants $\beta' > 1, \xi'>0,$ and $1/2\leq \eta' < 1$, such that for every $z\geq 1$ and every integer $n\geq 1$:
	\begin{equation}
	\pp\left(\sum_{t=1}^n Y_t - n f^* \geq n^{\eta'}z  \right) \leq \frac{\beta'}{z^{\xi'}}
	\quad\text{and}\quad
	\prob{\sum_{t=1}^n Y_t - n f^* \leq -n^{\eta'}z } \leq \frac{\beta'}{z^{\xi'}},
	\end{equation}
	where $\eta' = \frac{\frac{\alpha}{\xi (1-\eta)}+d'+\frac{1}{1-\eta}}{1+d'+\frac{1}{1-\eta}}$, $\xi' = (\alpha-3)/2$, and $\beta'>1$ depends on $\alpha,\beta,\eta,\xi$ and $\bar{H}$.
\end{theorem}

Theorem~\ref{thm:bandit} states the properties of the regret induced by the enhanced HOO algorithm (Algorithms~\ref{alg:HOO_query} and~\ref{alg:HOO_update}) for a non-stationary continuous-armed bandit problem, which may be of independent interest. If the rewards of the non-stationary bandit satisfy certain convergence rate and concentration conditions, then the regret of our algorithm also enjoys the same convergence rate and similar concentration guarantees. We can verify that our configuration of the parameters $\alpha^{(d)},\xi^{(d)},\eta^{(d)},\ 0\leq d \leq D-1$ in Theorem~\ref{thm:mcts} satisfy the requirements of Theorem~\ref{thm:bandit}. Therefore, using this theorem we can propagate the convergence result on one level of MCTS to the level above it. By applying Theorem~\ref{thm:bandit} recursively, we can establish the convergence result of the value function estimate for the root node of MCTS. 

%\begin{remark}
	In addition to the technical difficulty of analyzing the regret of HOO~\citep{bubeck2011x}, we have to address the challenges raised by the non-stationary rewards and bounded depth of HOO tree. The results are formally established as a sequence of lemmas in Appendix~\ref{appendix:lemmas}. 
%\end{remark}

%% file: 5simulations.tex
\section{Simulations}\label{sec:simulations}
In this section, we empirically evaluate the performance of \texttt{POLY-HOOT} on several classic control tasks. We have chosen three benchmark tasks from OpenAI Gym~\citep{openai}, and extended them to the continuous-action settings as necessary. These tasks include CartPole, Inverted Pendulum Swing-up, and LunarLander. CartPole is relatively easy, so we have also modified it to a more challenging one, CartPole-IG, with an increased gravity value. This new setting requires smoother actions, and bang-bang control strategies easily cause the pole to fall due to the increased inertia. 

We compare the empirical performance of \texttt{POLY-HOOT} with three other continuous MCTS algorithms, including UCT~\citep{kocsis2006bandit} with manually discretized actions, Polynomial Upper Confidence Trees (PUCT) with progressive widening~\citep{auger2013continuous}, and the original implementation of HOOT~\citep{mansley2011sample} with a logarithmic bonus term. Their average rewards and standard deviations on the above tasks are shown in Table~\ref{tbl:1}. The results are averaged over $40$ runs. The detailed experiment settings as well as additional experiment results can be found in Appendix~\ref{appendix:simulations}. 

\begin{table}[!htbp]
	\centering
	\scalebox{0.9}{
	\begin{tabular}{ccccc}
		\hline
		& CartPole & CartPole-IG & Pendulum & LunarLander \\ \hline
		discretized-UCT &     77.85 $\pm $ 0.0    &       69.39     $\pm $  	6.63     &     -109.68 $\pm$ 0.29    &       -57.95 $\pm$ 	77.36     \\
		PUCT            &     77.85 $\pm$  0.0   &      71.48 $\pm $ 8.27          &      -109.64 $\pm$ 	0.25    &       -43.05 $\pm$ 	80.25    \\
		HOOT            &    77.85 $\pm $  0.0   &      77.85    $\pm$   0.0      &  -109.50  $\pm$ 0.35      &      -23.37  $\pm$ 76.46     \\
		\texttt{POLY-HOOT}       &   77.85  $\pm $  0.0   &      77.85    $\pm$   0.0      &   -109.43 $\pm$ 0.25     &         -3.02 $\pm$ 	44.41    \\ \hline
	\end{tabular}
}
	\caption{Empirical performances on classic control tasks}\label{tbl:1}
\end{table}

\begin{table}[!htbp]
	
	\vspace{-0.2cm}\centering
	\scalebox{0.91}{
		\begin{tabular}{@{\hspace{-0.001cm}}c@{\hspace{-0.001cm}}cccccccc}
			\hline
			Algorithm & discretized-UCT& PUCT &HOOT & {\small $\bar{H}=2$} & {\small $\bar{H}=4$} & {\small $\bar{H}=6$} & {\small $\bar{H}=8$} & {\small $\bar{H}=10$} \\ %\hline
			Reward &69.03&70.79  &77.85& 42.45 & 48.54 & 63.27 & 77.85 & 77.85\\
			Time per decision (s) &0.950&0.305&1.173&   0.054      &      0.149            &   0.610      &  1.030  & 1.057         \\ \hline
		\end{tabular}
	}
	\caption{Time per decision on CartPole-IG}\label{tbl:3}\vspace{-0.4cm}
\end{table}

As we can see from Table~\ref{tbl:1}, all four algorithms achieve optimal rewards on the easier CartPole task. However, for the CartPole-IG task with increased  gravity, discretized-UCT and PUCT do not achieve the optimal performance, because their actions, either sampled from a uniform grid or sampled completely randomly, are not smooth enough to handle the larger momentum. In the Pendulum task, the four algorithms have similar performance, although HOOT and \texttt{POLY-HOOT} perform slightly better. Finally, on LunarLander, HOOT and \texttt{POLY-HOOT} achieve much better performance. This task has a high-dimensional action space, making it difficult for discretized-UCT and PUCT to sample actions at fine granularity. Also note that \texttt{POLY-HOOT} significantly outperforms HOOT. We believe the reason is that this task, as detailed in Appendix~\ref{appendix:simulations}, features a deeper search depth and sparse but large positive rewards. This causes a more severe non-stationarity issue of rewards within the search tree, which is better handled by \texttt{POLY-HOOT} with a polynomial bonus term than by HOOT, as our theory suggests. This demonstrates the superiority of \texttt{POLY-HOOT} in dealing with complicated continuous-space tasks with higher dimensions and deeper planning depth. We would also like to remark that the high standard deviations in this task are mostly due to the reward structure of the task itself---the agent either gets a large negative reward (when the lander crashes) or a large positive reward (when it lands on the landing pad) in the end.

We also empirically evaluate the time complexity of the algorithms. Table~\ref{tbl:3} shows the time needed by each algorithm to make a single decision on CartPole-IG. For \texttt{POLY-HOOT}, we further test its computation time with different values of $\bar{H}$ (the maximum depth of the HOO tree), which is an important hyper-parameter to balance the trade-off between optimality and time complexity. All tests are  averaged over 10 (new) runs on a laptop with an Intel Core i5-9300H CPU. We can see that \texttt{POLY-HOOT} requires slightly more computation than discretized-UCT and PUCT as the cost of higher rewards, but it is still more time-efficient than HOOT because of the additional depth limitation.

%% file: 6supplementary.tex
\appendix
\clearpage

\onecolumn

~\\
\centerline{{\fontsize{13.5}{13.5}\selectfont \textbf{Supplementary Materials for ``\texttt{POLY-HOOT}: Monte-Carlo Planning}}}

\vspace{6pt}
 \centerline{\fontsize{13.5}{13.5}\selectfont \textbf{
  in Continuous Space MDPs with Non-Asymptotic Analysis''}}
 \vspace{10pt}

\section{Algorithm Details}\label{appendix:algorithm}
In the following, we provide the details of the functions $HOO\_query$ and $HOO\_update$ that are utilized in Algorithm~\ref{alg:POLY-HOOT}.

\begin{algorithm}[H]
	%\SetAlgoLined
	\textbf{Input:} depth in MCTS $d$, state $s$, and round $t$.
	
	\textbf{Output:} action to take $a$.
	
	\textbf{Parameters:} maximum depth $\bar{H}$ allowed in HOO.
	
	\If{state $s$ has never been visited at MCTS depth $d$}
	{
		Initialize HOO agent at state $s$ and depth $d$: $\mathcal{T}\gets \{(0,1)\}$ and $B_{1,2},B_{2,2}\gets \infty$\;
	}
	\Else
	{
		$\mc{T}\gets $ the HOO agent constructed at state $s$ and depth $d$ previously\;
	}
	$(h,i) \gets (0,1)$\;
	Initialize HOO path in the current round: $P_t \gets \{(h,i)\}$\;
	\While{$(h,i)\in\mc{T}$}
	{
		\If{$B_{h+1,2i-1} > B_{h+1,2i}$}
		{
			$(h,i)\gets (h+1,2i-1)$\;
		}
		\Else
		{
			$(h,i)\gets (h+1,2i)$\;
		}
		$P_t \gets P_t\cup \{(h,i)\}$
	}
	$(H,I)\gets (h,i)$\;
	\If{$H\leq \bar{H}$}
	{
		Choose arbitrary arm $X$ in $\mc{P}_{H,I}$\;
		$A_{H,I} = X$\;
		\tcp{Associate the chosen action $X$ with the node $(H,I)$.}
		$\mc{T} \gets  \mc{T}\cup \{(H,I)\}$\;
		$B_{H+1,2 I-1},B_{H+1,2 I} \gets \infty$\;
		\Return $X$\;
	}
	\Else
	{
		\tcp{We reached the maximum depth and should not explore new actions.}
		$(H,I) \gets (H-1, \ceil{I/2})$\; 
		\Return $A_{H,I}$. 
	}
	
	\caption{HOO\_query}\label{alg:HOO_query}
\end{algorithm}

\begin{algorithm}[H]
	%\SetAlgoLined
	\textbf{Input:} depth in MCTS $d$, state $s$, and bandit reward $Y$ at round $t$.
	
	\textbf{Parameters:}  $\alpha^{(d)}, \xi^{(d)},\eta^{(d)},\nu_1$ and $\rho$.
	
	$\alpha, \xi, \eta \gets \alpha^{(d)}, \xi^{(d)}, \eta^{(d)}$\;
	\ForEach{$(h,i)$ in $P_t$}
	{
		$T_{h,i}\gets T_{h,i} + 1$\;
		$\widehat{\mu}_{h, i} \gets \left(1-1 / T_{h, i}\right) \widehat{\mu}_{h, i}+Y / T_{h, i}$\;
	}
	\ForEach{$(h,i)$ in $\mc{T}$}
	{
		$U_{h, i} \gets \widehat{\mu}_{h, i}+t^{\alpha/\xi} T_{h,i}^{\eta-1}+\nu_{1} \rho^{h}$\;
	}
	$\mc{T'} \gets \mc{T}$\;
	\While{$\mc{T'}\neq \{(0,1)\}$}
	{
		$(h,i)\gets $ an arbitrary leaf node of $\mc{T}'$\;
		$B_{h, i} \leftarrow \min \left\{U_{h, i}, \max \left\{B_{h+1,2 i-1}, B_{h+1,2 i}\right\}\right\}$\;
		$\mc{T}' \gets \mc{T}' \backslash \{(h,i)\}$\;
	}
	\caption{HOO\_update}\label{alg:HOO_update}
\end{algorithm}

\section{Proof of Theorem~\ref{thm:bandit}}\label{appendix:bandit}
Let $R_n = \sum_{t=1}^n (f^*-Y_t)$ denote the regret of Algorithms~\ref{alg:HOO_query} and~\ref{alg:HOO_update} with the depth limitation $\bar{H}$. We define the following notations that are similar to~\cite{bubeck2011x}. First, let $I_h$ denote the set of nodes at depth $h$ that are $2\nu_1\rho^h$-optimal, i.e., the set of nodes $(h,i)$ that satisfy $f^*_{h,i} \geq f^*- 2\nu_1\rho^h$, where $f_{h,i}^*\triangleq \sup_{x\in\mc{P}_{h,i}}f(x)$. For $h\geq 1$, let $J_h$ denote the set of nodes at depth $h$ that are not in $I_h$ but whose parents are in $I_{h-1}$ (i.e., they are not $2\nu_1\rho^h$-optimal themselves but their parents are $2\nu_1\rho^{h-1}$-optimal). Finally, define $\mc{X}_\varepsilon \triangleq \{x\in X:f(x)\geq f^* - \varepsilon\}$ to be the set of arms that are $\varepsilon$-close to optimal. 

Let $(H_t,I_t)$ denote the node that is selected by the bandit algorithm at time $t$. Note that with the depth limitation $\bar{H}$ it is possible that the nodes on depth $\bar{H}$ might be played more than once at different rounds. The nodes above depth $\bar{H}$ (i.e., $H_t< \bar{H}$), on the other hand, are played only once and the random variables $(H_t,I_t)$ are not the same for different values of $t$.  Let $\mc{L} = \{(H_t,I_t):H_t = \bar{H}\}$ denote the set of nodes on depth $\bar{H}$ that have been played. Let $H\geq 1$ be a constant integer whose value will be specified later, and without loss of generality we assume $\bar{H} > H$. We partition the nodes in the HOO tree $\mc{T}$ above depth $\bar{H}$ into three parts $\mc{T}\backslash\mc{L} = \mc{T}_1\cup\mc{T}_2\cup\mc{T}_3$. Let $\mc{T}_1$ be the set of nodes above depth $\bar{H}$ that are descendants of nodes in $I_H$. By convention, a node itself is also considered as a descendant of its own, so we also have $I_H\subseteq \mc{T}_1$. Let $\mathcal{T}_{2}=\cup_{0 \leq h<H} I_{h}$. Finally, let $\mc{T}_3$ be the set of nodes above depth $\bar{H}$ that are descendants of nodes in $\cup_{0 \leq h\leq H} J_{h}$. We can verify that $\mc{T}_1\cup\mc{T}_2\cup\mc{T}_3\cup\mc{L}$ covers all the nodes in $\mc{T}$. 

Similarly, we also decompose the regret according to the selected node $(H_t,I_t)$ into four parts: $R_n = R_{n,1} + R_{n,2} + R_{n,3} + R_\mc{L}$, where $R_{n, i}=\sum_{t=1}^{n}\left(f^{*}-Y_t\right) \mathbb{I}_{\left\{\left(H_{t}, I_{t}\right) \in \mathcal{T}_{i}\right\}}$ and $R_{\mc{L}}=\sum_{t=1}^{n}\left(f^{*}-Y_t\right) \mathbb{I}_{\left\{\left(H_{t}, I_{t}\right) \in \mathcal{L}\right\}}$. In the following, we analyze each of the four parts individually. We start with the concentration property and then the convergence results. 

To proceed further, we first need to state several definitions that are useful throughout. These definitions come from~\cite{bubeck2011x}, with similar ideas introduced earlier in~\cite{auer2007improved}. We reproduce these definitions here for completeness. 

\begin{definition}(Packing number)
	The $\epsilon$-packing number $\mc{N}(\mc{X},\ell, \epsilon)$ of $\mc{X}$ w.r.t the dissimilarity $\ell$ is the largest integer $k$ such that there exists $k$ disjoint $\ell$-open balls with radius $\epsilon$ contained in $\mc{X}$. 
\end{definition}

\begin{definition}(Near-optimality dimension)
	For $c>0$, the near-optimality dimension of $f$ w.r.t $\ell$ is
	\[
	\max \left\{0, \limsup _{\varepsilon \rightarrow 0} \frac{\ln \mathcal{N}\left(\mc{X}_{c \varepsilon}, \ell, \varepsilon\right)}{\ln \left(\varepsilon^{-1}\right)}\right\}.
	\]
\end{definition}

\begin{definition}\label{dfn:d'}
	Let $d$ be the $4\nu_1/\nu_2-$near-optimality dimension of $f$ w.r.t $\ell$. We use $d'$ to denote any value such that $d' > d$.
\end{definition}

\begin{definition}\label{dfn:optimalnode}
	Given the limit of the mean-payoff function $f$ of a HOO agent, we assume without loss of generality that $(0,1), (1,i_1^*), (2, i_2^*),\dots, (\bar{H},i_{\bar{H}}^*)$ is an optimal path, i.e., $\Delta_{h, i_h^*} = 0,\forall h\geq 1$. We define the nodes $(h, i_h^*)$ on the optimal path as optimal nodes, and the other nodes as suboptimal nodes. 
\end{definition}

Our proof will also rely on several lemmas that we state and prove in Appendix~\ref{appendix:lemmas}.

\subsection{Regret from $\mc{T}_1$}
Any node in $I_H$ is by definition $2\nu_1\rho^H$-optimal. By Lemma~\ref{lemma:3}, the domain of $I_H$ lies in $\mc{X}_{4\nu_1\rho^H}$. Since the descendants of $I_H$ cover a domain that is a subset of the domain of $I_H$, we know the descendants of $I_H$ also lie in the domain of $\mc{X}_{4\nu_1\rho^H}$, and hence $\sum_{t=1}^{n}\left(f^{*}-f\left(X_{t}\right)\right) \mathbb{I}_{\left\{\left(H_{t}, I_{t}\right)\in \mathcal{T}_{1}\right\}} \leq  4\nu_1 \rho^H n$. Let $n_1 = |\mc{T}_1|$ we then have for every $z\geq 1$, 
\[
\begin{aligned}
&\pp\left( R_{n,1} \geq zn^{\eta}  + 4\nu_1 \rho^H n\right)\\
=& \prob{\sum_{t=1}^{n}\left(f^{*}-Y_t\right) \mathbb{I}_{\left\{\left(H_{t}, I_{t}\right) \in \mathcal{T}_{1}\right\}} \geq zn^{\eta}  + 4\nu_1 \rho^H n}\\
=& \prob{\sum_{t=1}^{n}\left(f^{*}-f\left(X_{t}\right)\right) \mathbb{I}_{\left\{\left(H_{t}, I_{t}\right)\in \mathcal{T}_{1}\right\}} + \sum_{t=1}^{n}\left(f(X_t)-Y_t\right) \mathbb{I}_{\left\{\left(H_{t}, I_{t}\right) \in \mathcal{T}_{1}\right\}} \geq zn^{\eta}  + 4\nu_1 \rho^H n}\\
\leq & \sum_{t=1}^{n_1} \prob{f(\tilde{X}_t)-\tilde{Y}_t  \geq \frac{z}{n_1}n^{\eta}  }\\
\leq & \frac{n_1^{\xi+1} \beta}{z^{\xi}}\leq \frac{c_1^{\xi+1}\beta}{z^{\alpha-3}},
\end{aligned}
\]
where $\tilde{X}_t$ denotes the $t$-th arm pulled in $\mc{T}_1$, and $\tilde{Y}_t$ denotes its corresponding reward. Note that in the first inequality we used the fact that $\sum_{t=1}^{n}\left(f^{*}-f\left(X_{t}\right)\right) \mathbb{I}_{\left\{\left(H_{t}, I_{t}\right)\in \mathcal{T}_{1}\right\}} \leq  4\nu_1 \rho^H n$. In the second inequality we used the union bound. In the third inequality we applied the concentration property of the bandit problem~\eqref{eqn:concentration} with $n=1$. Notice that we can only use the concentration property when the requirement $\frac{z}{n_1}\geq 1$ is satisfied, but when $\frac{z}{n_1} < 1$, the inequality also trivially holds because $\frac{n_1^{\xi+1}\beta}{z^{\xi}}  > 1$. The last step holds because $\alpha - 3 < \alpha < \xi (1-\eta) < \xi$, and $c_1\geq 1$ is a constant that upper bounds $n_1$ (since $\mc{T}$ is a binary tree with limited depth, one trivial upper bound would be the number of nodes in $\mc{T}$, which does not depend on $n$ and $z$). Also notice that the inequality above trivially holds when $0 < z < 1$, because $\beta > 1, \alpha-3 > 0$ and hence $\frac{\beta}{z^{\alpha -3}} > 1$ is an upper bound for any probability value.

Let $\lambda = \frac{\frac{\alpha}{\xi (1-\eta)}-1}{1+d'+\frac{1}{1-\eta}}$, and we know $\lambda < 0$ because $\alpha < \xi (1-\eta)$. We then choose the value for $H$ such that $\rho^H = n^\lambda$; then, $4\nu_1 \rho^H n$ is of the order of $n^{\lambda+1}$. We further have $n^{\lambda+1} > n^\eta$ since $\alpha \geq \xi \eta (1-\eta)$. Let $c_2\geq 1$ be a constant such that $c_2 n^{\lambda+1} \geq c_2^{1/2} n^\eta + 4 \nu_1 n^{\lambda+1},\forall n \geq 1$. Such a constant always exists because $c_2^{1/2} < c_2$ and $n^\eta < n^{\lambda+1}$. Then it is easy to see that $z n^{\lambda+1} \geq z^{1/2} n^\eta + 4 \nu_1 n^{\lambda+1},\forall n \geq 1$  also holds for any $z\geq c_2$. Therefore, we have the following property:
\begin{equation}\label{eqn:T1}
\pp\left(R_{n,1} \geq  z n^{\lambda +1}  \right)  \leq \frac{c_1^{\xi+1 }c_2^{\alpha-3} \beta}{z^{(\alpha-3)/2}},\ \forall z \geq 1.
\end{equation}
To see this, first suppose that $z \geq c_2$; then, $z n^{\lambda+1}  \geq z^{1/2} n^\eta + 4 \nu_1 n^{\lambda+1} ,\forall n \geq 1$ and since $c_2\geq 1$, we have $\pp\left(R_{n,1} \geq  z n^{\lambda+1} \right) \leq \pp\left(R_{n,1} \geq  \frac{z^{1/2}}{c_2} n^{\eta}+4 \nu_{1} \rho^{H} n \right) \leq \frac{c_1^{\xi+1}c_2^{\alpha-3} \beta}{z^{(\alpha-3)/2}}$. On the other hand, if $1\leq z < c_2$, then the inequality~\eqref{eqn:T1} trivially holds, because $c_2^{\alpha-3} > z^{\alpha-3} \geq  z^{(\alpha-3)/2}$ and $\beta > 1, c_1\geq 1$, making the RHS greater than $1$. The other side of the concentration inequality follows similarly and is omitted here.

\subsection{Regret from $\mc{T}_2$}
For $h\geq 0$, any node $(h,i)\in\mc{T}_2$ by definition belongs to $I_h$ and is hence $2\nu_1 p^h$-optimal. Therefore, $\sum_{t=1}^{n}\left(f^{*}-f\left(X_{t}\right)\right) \mathbb{I}_{\left\{\left(H_{t}, I_{t}\right)\in \mathcal{T}_{2}\right\}} \leq \sum_{h=0}^{H-1} 4 \nu_{1} \rho^{h}\left|I_{h}\right| \leq 4 c_3 \nu_{1} \nu_{2}^{-d^{\prime}} \sum_{h=0}^{H-1} \rho^{h\left(1-d^{\prime}\right)}$, where the last step uses the fact that $|I_h|\leq c_3\left(\nu_{2} \rho^{h}\right)^{-d^{\prime}}$ for some constant $c_3$ (Lemma~\ref{lemma:I_h} in Appendix~\ref{appendix:lemmas}). We then have the following convergence result:
\begin{equation}\label{eqn:convergence2}
\mathbb{E}\left[R_{n, 2}\right] \leq  4 c_3 \nu_{1} \nu_{2}^{-d^{\prime}} \sum_{h=0}^{H-1} \rho^{h\left(1-d^{\prime}\right)}.
\end{equation}

Let $n_2 = |\mc{T}_2|$; then for every $z\geq 1$, we have
\[
\begin{aligned}
&\pp\left( R_{n,2} \geq zn^{\eta}  + 4c_3\nu_1\nu_2^{-d'}\sum_{h=0}^{H-1}\rho^{h(1-d')} \right)\\
=& \prob{\sum_{t=1}^{n}\left(f^{*}-Y_t\right) \mathbb{I}_{\left\{\left(H_{t}, I_{t}\right) \in \mathcal{T}_{2}\right\}} \geq zn^{\eta}  + 4c_3\nu_1\nu_2^{-d'}\sum_{h=0}^{H-1}\rho^{h(1-d')} }\\
=& \pp\Bigg(\sum_{t=1}^{n}\left(f^{*}-f\left(X_{t}\right)\right) \mathbb{I}_{\left\{\left(H_{t}, I_{t}\right)\in \mathcal{T}_{2}\right\}} + \sum_{t=1}^{n}\left(f(X_t)-Y_t\right) \mathbb{I}_{\left\{\left(H_{t}, I_{t}\right) \in \mathcal{T}_{2}\right\}}\\
&\qquad \qquad \geq zn^{\eta}  + 4c_3\nu_1\nu_2^{-d'}\sum_{h=0}^{H-1}\rho^{h(1-d')} \Bigg)\\
\leq & \prob{\sum_{t=1}^{n}\left(f(X_t)-Y_t\right) \mathbb{I}_{\left\{\left(H_{t}, I_{t}\right) \in \mathcal{T}_{2}\right\}} \geq zn^{\eta}  }\\
\leq & \frac{n_2^{\xi+1} \beta}{z^{\xi}}\leq \frac{c_4^{\xi+1}\beta}{z^{\alpha-3}},
\end{aligned}
\]
where the first inequality uses the fact that  $\sum_{t=1}^{n}\left(f^{*}-f\left(X_{t}\right)\right) \mathbb{I}_{\left\{\left(H_{t}, I_{t}\right)\in \mathcal{T}_{2}\right\}}  \leq 4 c_3 \nu_{1} \nu_{2}^{-d^{\prime}} \sum_{h=0}^{H-1} \rho^{h\left(1-d^{\prime}\right)}$, and $c_4$ is a constant not depending on $n$ and $z$ that upper bounds $n_2$, similar to the proof in $\mc{T}_1$. Again, this inequality also trivially holds for $0 < z<1$.

Since there exists a constant $c_5$ that 
$$
\begin{aligned} 
\sum_{h=0}^{H-1} \rho^{h\left(1-d^{\prime}\right)} &\leq c_5 \rho^{H(1-d')} \leq c_5\rho^{-H(d' + \frac{1}{1-\eta})}\leq c_5 \rho^{-H(d'+\frac{1}{1-\eta})}n^{\frac{\alpha}{\xi(1-\eta)}}\leq c_5 n^{\lambda+1},
\end{aligned} 
$$
we know $4 c_3 \nu_{1} \nu_{2}^{-d^{\prime}} \sum_{h=0}^{H-1} \rho^{h\left(1-d^{\prime}\right)}$ is upper bounded by the order of $n^{\lambda+1}$. Again, since $n^{\lambda+1}>n^\eta$, there always exists a constant $c_6\geq 1$ such that for any $z\geq c_6$, $z n^{\lambda+1} \geq z^{1/2} n^\eta + 4 c_3 \nu_{1} \nu_{2}^{-d^{\prime}} \sum_{h=0}^{H-1} \rho^{h\left(1-d^{\prime}\right)},\forall n\geq 1$. Therefore, we have
\begin{equation}\label{eqn:T2}
\pp\left(R_{n,2} \geq  z n^{\lambda +1} \right)  \leq \frac{c_4^{\xi+1}c_6^{\alpha-3} \beta}{z^{(\alpha-3)/2}},\ \forall z \geq 1.
\end{equation}
To see this, again, first suppose that $z\geq c_6$, then $z n^{\lambda+1} \geq z^{1/2} n^\eta + 4 c_3 \nu_{1} \nu_{2}^{-d^{\prime}} \sum_{h=0}^{H-1} \rho^{h\left(1-d^{\prime}\right)}$, and hence $\pp\left(R_{n,2} \geq  z n^{\lambda +1} \right) \leq \pp\left(R_{n,2} \geq  \frac{z^{1/2}}{c_6} n^\eta + 4 c_3 \nu_{1} \nu_{2}^{-d^{\prime}} \sum_{h=0}^{H-1} \rho^{h\left(1-d^{\prime}\right)} \right) \leq \frac{c_4^{\xi+1}c_6^{\alpha-3}\beta}{z^{(\alpha-3)/2}}$. If on the other hand $1 \leq z<c_6$, then inequality~\eqref{eqn:T2} trivially holds because the RHS is greater than $1$.

\subsection{Regret from $\mc{T}_3$}
For any node $(h,i)\in\mc{T}_3$, since the parent of any $(h,i)\in J_h$ is in $I_{h-1}$, we know by Lemma~\ref{lemma:3} that the domain of $(h,i)$ is in $\mc{X}_{4\nu_1 \rho^{h-1}}$. Further, for any $u\geq A_{h,i}(n) = \ceil{ \left(\frac{2n^{\alpha/\xi}}{\Delta_{h,i}-\nu_1\rho^h}\right)^{\frac{1}{1-\eta}}}$ and $z\geq 1$, we know from inequality~\eqref{eqn:pt>u} that $\prob{T_{h,i}(n)>zu}\leq \frac{(zu-1)^{3-\alpha}}{n} + \frac{(zu-1)^{3-\alpha}}{\alpha-3}\leq z^{3-\alpha}(u-1)^{3-\alpha}\left(\frac{1}{n} + \frac{1}{\alpha-3}\right)$. Since $\Delta_{h,i} > 2\nu_1\rho^h$, we know $A_{h,i}(n) \leq \ceil{\left(\frac{2n^{\alpha/\xi}}{\nu_1 \rho^h}\right)^{\frac{1}{1-\eta}}}$. Then
for any $u> \left(\frac{2n^{\alpha/\xi}}{\nu_1 \rho^h}\right)^{\frac{1}{1-\eta}}$, 
\[
\begin{aligned}
&\prob{\sum_{t=1}^{n}\left(f^{*}-f\left(X_{t}\right)\right) \mathbb{I}_{\left\{\left(H_{t}, I_{t}\right)\in \mathcal{T}_{3}\right\}} \geq \sum_{h=1}^H 4\nu_1 \rho^{h-1} \sum_{(h,i)\in\mc{T}_3}zu}\\
\leq&  \prob{\sum_{h=1}^H 4\nu_1 \rho^{h-1} \sum_{(h,i)\in\mc{T}_3}T_{h,i}(n) \geq \sum_{h=1}^H 4\nu_1 \rho^{h-1} \sum_{(h,i)\in\mc{T}_3}zu}\\
\leq&  \sum_{h=1}^H\prob{ \sum_{(h,i)\in\mc{T}_3}T_{h,i}(n) \geq \sum_{(h,i)\in\mc{T}_3}zu}\\
\leq &\sum_{h=1}^H |J_h|z^{3-\alpha}(u-1)^{3-\alpha}\left( \frac{1}{n} + \frac{1}{\alpha-3}\right)\\
\leq & 2C\nu_2^{-d'}\sum_{h=1}^H \rho^{-(h-1)d'}z^{3-\alpha}(u-1)^{3-\alpha}\left( \frac{1}{n} + \frac{1}{\alpha-3}\right),
\end{aligned}
\]
where in the last step we used the fact that $|J_h|\leq 2|I_{h-1}|\leq 2c_2\left(\nu_{2} \rho^{h-1}\right)^{-d^{\prime}}$, because the parent of any node in $J_h$ is in $I_{h-1}$. Since $\alpha > 3$, we know $2c_2\nu_2^{-d'}\sum_{h=1}^H \rho^{-(h-1)d'}(u-1)^{3-\alpha}\left( \frac{1}{n} + \frac{1}{\alpha-3}\right)$ decreases polynomially in $n$, and hence there exists some constant $c_7>1$, such that $2c_2\nu_2^{-d'}\sum_{h=1}^H \rho^{-(h-1)d'} (u-1)^{3-\alpha}\left( \frac{1}{n} + \frac{1}{\alpha-3}\right) \leq c_7,\ \forall n\geq 1$. Therefore, for any $z\geq 1$, 
\[
\prob{\sum_{t=1}^{n}\left(f^{*}-f\left(X_{t}\right)\right) \mathbb{I}_{\left\{\left(H_{t}, I_{t}\right)\in \mathcal{T}_{3}\right\}} \geq \sum_{h=1}^H 4\nu_1 \rho^{h-1} \sum_{(h,i)\in\mc{T}_3}zu} \leq c_7 z^{3-\alpha}.
\]
Let $n_3 = |\mc{T}_3|$, and let $\mb{I}_{\{\cdot\}}$ denote $\mathbb{I}_{\left\{\left(H_{t}, I_{t}\right) \in \mathcal{T}_{3}\right\}}$ for short; then for every $z\geq 1$, we have
\[
\begin{aligned}
&\pp\left( R_{n,3} \geq zn^{\eta}  + \sum_{h=1}^H 4\nu_1 \rho^{h-1} \sum_{(h,i)\in\mc{T}_3}zu \right)\\
=& \prob{\sum_{t=1}^{n}\left(f^{*}-Y_t\right) \mathbb{I}_{\left\{\left(H_{t}, I_{t}\right) \in \mathcal{T}_{3}\right\}} \geq zn^{\eta}  + \sum_{h=1}^H 4\nu_1 \rho^{h-1} \sum_{(h,i)\in\mc{T}_3}zu }\\
=& \prob{\sum_{t=1}^{n}\left(f^{*}-f\left(X_{t}\right)\right) \mb{I}_{\{\cdot\}} + \sum_{t=1}^{n}\left(f(X_t)-Y_t\right) \mb{I}_{\{\cdot\}} \geq zn^{\eta}  + \sum_{h=1}^H 4\nu_1 \rho^{h-1} \sum_{(h,i)\in\mc{T}_3}zu }\\
\leq & \prob{\sum_{t=1}^{n}\left(f(X_t)-Y_t\right) \mb{I}_{\{\cdot\}} \geq zn^{\eta}  } + \prob{\sum_{t=1}^{n}\left(f^{*}-f\left(X_{t}\right)\right) \mb{I}_{\{\cdot\}}\geq \sum_{h=1}^H 4\nu_1 \rho^{h-1} \sum_{(h,i)\in\mc{T}_3}zu}\\
= & \frac{n_3^{\xi+1} \beta}{z^{\xi}}+\prob{\sum_{t=1}^{n}\left(f^{*}-f\left(X_{t}\right)\right) \mb{I}_{\{\cdot\}} \geq \sum_{h=1}^H 4\nu_1 \rho^{h-1} \sum_{(h,i)\in\mc{T}_3}zu}\\
\leq & \frac{c_8^{\xi+1}\beta}{z^{\xi}} + c_7z^{3-\alpha}\leq \frac{c_8^{\xi+1}\beta+c_7}{z^{\alpha-3}},
\end{aligned}
\]
where as before $c_8$ is a constant not depending on $n$ and $z$ that upper bounds $n_3$, and in the last step we used the fact that $\alpha - 3 < \alpha < \xi (1-\eta ) < \xi$. 

Once again, since $\sum_{h=1}^H 4\nu_1 \rho^{h-1} \sum_{(h,i)\in\mc{T}_3}u$ is upper bounded by the order of $n^{\lambda+1}$, there exists a constant $c_9\geq 1$ such that for any $z\geq c_9$, $z n^{\lambda+1} \geq z^{1/2} n^\eta +\sum_{h=1}^H 4\nu_1  \rho^{h-1} \sum_{(h,i)\in\mc{T}_3}z^{1/2}u,\forall n\geq 1$. Therefore, we have 
\begin{equation}\label{eqn:T3}
\pp\left(R_{n,3} \geq  z n^{\lambda +1} \right)  \leq \frac{c_9^{\alpha-3} (c_8^{\xi+1}\beta + c_7)}{z^{(\alpha-3)/2}},\ \forall z \geq 1,
\end{equation}
due to exactly the same logic as in $\mc{T}_1$ and $\mc{T}_2$, by discussing the two cases $z \geq c_9$ and $1\leq z < c_9$.
\subsection{Regret from $\mc{L}$}
Recall that $\mc{L}$ is the set of nodes that are played on depth $\bar{H}$. We divide the nodes in $\mc{L}$ into two parts $\mc{L} = \mc{L}_1\cup \mc{L}_3$, in analogy to $\mc{T}_1$ and $\mc{T}_3$ in $\mc{T}\backslash\mc{L}$. Let $\mc{L}_1$ be the set of nodes on depth $\bar{H}$ that are descendants of nodes in $I_H$, and let $\mc{L}_3$ be the set of nodes in $\mc{L}$ that are descendants of nodes in $\cup_{0 \leq h\leq H} J_{h}$. By the assumption that $\bar{H} > H$, there is no counterpart of $\mc{T}_2=\cup_{0 \leq h<H}I_h$ in $\mc{L}$.

Similarly, we also decompose the regret from $\mc{L}$ according to the selected node $(H_t,I_t)$ into two parts: $R_\mc{L} = \tilde{R}_{n,1} + \tilde{R}_{n,3}$, where $\tilde{R}_{n, i}=\sum_{t=1}^{n}\left(f^{*}-Y_t\right) \mathbb{I}_{\left\{\left(H_{t}, I_{t}\right) \in \mathcal{L}_{i}\right\}}$. Analyzing the regret from $\mc{L}_1$ and $\mc{L}_3$ is almost the same as $\mc{T}_1$ and $\mc{T}_3$, with only one difference that each node in $\mc{L}$ might be played multiple times. We demonstrate with $\mc{L}_1$ in the following and the analysis for $\mc{L}_3$ naturally follows. 

Again, any node in $I_H$ is by definition $2\nu_1\rho^H$-optimal. By Lemma~\ref{lemma:3}, the domain of $I_H$ lies in $\mc{X}_{4\nu_1\rho^H}$, and we know the descendants of $I_H$ also lie in the domain of $\mc{X}_{4\nu_1\rho^H}$, satisfying $\sum_{t=1}^{n}\left(f^{*}-f\left(X_{t}\right)\right) \mathbb{I}_{\left\{\left(H_{t}, I_{t}\right)\in \mathcal{L}_{1}\right\}} \leq  4\nu_1 \rho^H n$. Let $\tilde{n}_1 = |\mc{L}_1|$. Let $\tilde{X}_1,\dots, \tilde{X}_{n_1}$ denote the arms pulled in $\mc{L}_1$ (we know from Algorithm~\ref{alg:HOO_query} that only one arm in a node will be played and associated with that node, and this arm will be played repeatedly thereafter). For $j=1,\dots, n_1$, define $K_j$ to be the total number of times arm $\tilde{X}_j$ has been played. Finally, let $\tilde{Y}_j^t\ (1\leq t \leq K_j)$ denote the corresponding reward when the $t$-th time arm $\tilde{X}_j$ is played. Then for every $z\geq 1$, 
\[
\begin{aligned}
&\pp\left( \tilde{R}_{n,1} \geq zn^{\eta}  + 4\nu_1 \rho^H n \right)\\
=& \prob{\sum_{t=1}^{n}\left(f^{*}-Y_t\right) \mathbb{I}_{\left\{\left(H_{t}, I_{t}\right) \in \mathcal{L}_{1}\right\}} \geq zn^{\eta}  + 4\nu_1 \rho^H n }\\
=& \prob{\sum_{t=1}^{n}\left(f^{*}-f\left(X_{t}\right)\right) \mathbb{I}_{\left\{\left(H_{t}, I_{t}\right)\in \mathcal{L}_{1}\right\}} + \sum_{t=1}^{n}\left(f(X_t)-Y_t\right) \mathbb{I}_{\left\{\left(H_{t}, I_{t}\right) \in \mathcal{L}_{1}\right\}} \geq zn^{\eta}  + 4\nu_1 \rho^H n }\\
\leq & \prob{\sum_{t=1}^{n}\left(f(X_t)-Y_t\right) \mathbb{I}_{\left\{\left(H_{t}, I_{t}\right) \in \mathcal{L}_{1}\right\}} \geq zn^{\eta}  }\\
\leq & \sum_{j=1}^{n_1} \prob{\sum_{t=1}^{K_j}\left(f(\tilde{X}_j)-\tilde{Y}^t_j\right)   \geq \frac{z}{\tilde{c}_1}K_j^{\eta}}\\
\leq & \frac{\tilde{c}_1^{\xi+1} \beta}{z^{\xi}}\leq \frac{\tilde{c}_1^{\xi+1}\beta}{z^{\alpha-3}},
\end{aligned}
\]
where $\tilde{c}_1\geq n_1$ is a constant that is independent of $n$ and $z$, and hence $\sum_{j=1}^{n_1} \frac{z}{\tilde{c}_1}K_j^\eta \leq \frac{z}{n_1}\sum_{j=1}^{n_1}n^\eta \leq zn^\eta$. Note that in the first inequality we used the fact that $\sum_{t=1}^{n}\left(f^{*}-f\left(X_{t}\right)\right) \mathbb{I}_{\left\{\left(H_{t}, I_{t}\right)\in \mathcal{T}_{1}\right\}} \leq  4\nu_1 \rho^H n$. In the second inequality, we used the union bound. In the third inequality we applied the concentration property of the bandit problem~\eqref{eqn:concentration} with $n=K_j$. Notice that we can only use the concentration property when the requirement $\frac{z}{\tilde{c}_1}\geq 1$ is satisfied, but when $\frac{z}{\tilde{c}_1} < 1$, the inequality also trivially holds because $\frac{\tilde{c}_1^{\xi+1}\beta}{z^{\xi}}  > 1$. The last step holds because $\alpha - 3 < \alpha < \xi (1-\eta) < \xi$. Also notice that the inequality above trivially holds when $0 < z < 1$, because $\beta > 1, \alpha-3 > 0$ and hence $\frac{\beta}{z^{\alpha -3}} > 1$ is an upper bound for any probability.

Similar to the analysis of $\mc{T}_1$, let $\tilde{c}_2\geq 1$ be a constant such that $\tilde{c}_2 n^{\lambda+1} \geq \tilde{c}_2^{1/2} n^\eta + 4 \nu_1 n^{\lambda+1},\forall n \geq 1$. Such a constant always exists because $\tilde{c}_2^{1/2} < \tilde{c}_2$ and $n^\eta < n^{\lambda+1}$. Then it is easy to see that $z n^{\lambda+1}  \geq z^{1/2} n^\eta + 4 \nu_1 n^{\lambda+1} ,\forall n \geq 1$  also holds for any $z\geq \tilde{c}_2$. Therefore, we have the following property:
\begin{equation}\label{eqn:T4}
\pp\left(\tilde{R}_{n,1} \geq  z n^{\lambda +1} \right)  \leq \frac{\tilde{c}_1^{\xi+1 }\tilde{c}_2^{\alpha-3} \beta}{z^{(\alpha-3)/2}},\ \forall z \geq 1.
\end{equation}
To see this, first suppose that $z \geq \tilde{c}_2$; then $z n^{\lambda+1}  \geq z^{1/2} n^\eta + 4 \nu_1 n^{\lambda+1} ,\forall n \geq 1$ and since $\tilde{c}_2\geq 1$, we have $\pp\left(\tilde{R}_{n,1} \geq  z n^{\lambda+1} \right) \leq \pp\left(\tilde{R}_{n,1} \geq  \frac{z^{1/2}}{\tilde{c}_2} n^{\eta}+4 \nu_{1} \rho^{H} n \right) \leq \frac{\tilde{c}_1^{\xi+1}\tilde{c}_2^{\alpha-3} \beta}{z^{(\alpha-3)/2}}$. On the other hand, if $1\leq z < \tilde{c}_2$, then the inequality~\eqref{eqn:T1} trivially holds, because $\tilde{c}_2^{\alpha-3} > z^{\alpha-3} \geq  z^{(\alpha-3)/2}$ and $\beta > 1, \tilde{c}_1\geq 1$, making the RHS greater than $1$. The other side of the concentration inequality follows similarly. This completes the analysis for $\tilde{R}_{n,1}$. 

Similarly, as for the regret from $\mc{L}_3$, we have the following result:
\begin{equation}\label{eqn:T5}
\pp\left(\tilde{R}_{n,3} \geq  z n^{\lambda +1} \right)  \leq \frac{\tilde{c}_9^{\alpha-3} (\tilde{c}_8^{\xi+1}\beta + \tilde{c}_7)}{z^{(\alpha-3)/2}},\ \forall z \geq 1,
\end{equation}
where again $\tilde{c}_7,\tilde{c}_8,\tilde{c}_9$ are constant independent of $n$ and $z$.

\subsection{Completing proof of concentration}
First, recall that the inequalities~\eqref{eqn:T1}\eqref{eqn:T2}\eqref{eqn:T3}\eqref{eqn:T4}\eqref{eqn:T5} still hold even when $0 < z < 1$. This is because the RHS of the inequalities will be greater than $1$, which is a trivial upper bound for a probability value. Putting together the bounds we got for each individual term, for every $z \geq 1$, we have
\[
\prob{R_n \geq z n^{\lambda+1}} \leq \sum_{i=1}^{3}\prob{R_{n,i}\geq \frac{z}{5}n^{\lambda+1} } + \sum_{i=1}^{2}\prob{\tilde{R}_{n,i}\geq \frac{z}{5}n^{\lambda+1} } \leq \frac{\beta'}{z^{(\alpha-3)/2}},
\]
where $\beta'>1$ is a constant independent of $n$ and $z$. Therefore, we have the desired concentration property:
\begin{equation}
\pp(\sum_{t=1}^n Y_t - n f^* \geq n^{\eta'}z  ) \leq \frac{\beta'}{z^{\xi'}},
\end{equation}
where $\xi' = (\alpha-3)/2, \eta' = \lambda+1 = \frac{\frac{\alpha}{\xi (1-\eta)}+d'+\frac{1}{1-\eta}}{1+d'+\frac{1}{1-\eta}}$, and $\beta'>1$ depends on $\alpha,\beta, \eta,\xi$ and $\bar{H}$. The other side of the concentration inequality follows similarly. 

\subsection{Convergence results}
We conclude with a convergence analysis of the regret. Let $R_n = \sum_{t=1}^n (f^*-Y_t)$ denote the regret of Algorithms~\ref{alg:HOO_query} and~\ref{alg:HOO_update} with the depth limitation $\bar{H}$. In the following, we proceed with the special case that there is only one optimal node on depth $\bar{H}$, i.e., there is only one node $(\bar{H}, I^*)$ on depth $\bar{H}$ with $\Delta_{\bar{H},I^*} \leq 2\nu_1\rho^{\bar{H}}$, which in turn implies $\mc{P}_{\bar{H},I^*}\subseteq \mc{X}_{4\nu_1 \rho^{\bar{H}}}$ (Lemma~\ref{lemma:3}). The regret of the general case with multiple optimal nodes is bounded by a constant multiple of this special case.

We partition the regret into three parts, but in a way that is slightly different from the previous concentration analysis. Let $R_n = R_{\mc{T}} + R_{n,1} + R_{n,3}$, where $R_{\mc{T}}$ denotes the regret above depth $\bar{H}$, $R_{n,1}$ denotes the regret from $\mc{L}_1$ (the set of nodes on depth $\bar{H}$ that are descendants of nodes in $I_H$), and $R_{n,3}$ denotes the regret from $\mc{L}_3$ (the set of nodes on depth $\bar{H}$ that are descendants of nodes in $\cup_{0\leq h\leq H}J_h$). Recall that the bandit rewards are bounded in $[-R, R]$. Then it is easy to see that $R_{\mc{T}}$ is bounded by a constant, denoted by $C_1$, because the number of nodes played above depth $\bar{H}$ is upper bounded by a constant independent of $n$. 

Now we consider $R_{n,1}$. Any node in $I_H$ is by definition $2\nu_1\rho^H$-optimal. By Lemma~\ref{lemma:3}, the domain of $I_H$ lies in $\mc{X}_{4\nu_1\rho^H}$, and we know the descendants of $I_H$ also lie in the domain of $\mc{X}_{4\nu_1\rho^H}$, satisfying $\sum_{t=1}^{n}\left(f^{*}-f\left(X_{t}\right)\right) \mathbb{I}_{\left\{\left(H_{t}, I_{t}\right)\in \mathcal{L}_{1}\right\}} \leq  4\nu_1 \rho^H n$. Let $\tilde{n}_1 = |\mc{L}_1|$, and let $\mb{I}_{\{\cdot\}}$ denote $\mb{I}_{\{(H_t,I_t) \in\mc{L}_1\}}$ for short. Then we have
\[
\begin{aligned}
\e{R_{n,1}} &= \e{\sum_{t=1}^n (f^* - Y_t) \mb{I}_{\{(H_t,I_t) \in\mc{L}_1\}}}\\
&=\e{\sum_{t=1}^n (f^* - f(X_t)) \mb{I}_{\{(H_t,I_t)\in\mc{L}_1\}}} + \e{\sum_{t=1}^n (f(X_t) - Y_t) \mb{I}_{\{(H_t,I_t) \in\mc{L}_1\}}}\\
&\leq 4n\nu_1\rho^{H}  + \e{\sum_{t=1}^n (f(X_t) - f_t(X_t)) \mb{I}_{\{\cdot\}}}+ \e{\sum_{t=1}^n (f_t(X_t) - Y_t) \mb{I}_{\{\cdot\}}}\\
&\leq 4n\nu_1\rho^{{H}}  + \sum_{t=1}^n \frac{C}{t^\zeta},
\end{aligned}
\]
where the last step holds due to the definition of the mean-payoff function that $\e{{Y}_t} = \e{f_t({X_t})}$ and the convergence property of $f_t$. Since $\sum_{t=1}^n \frac{1}{t^\zeta} \leq \int_{0}^{n}t^{-\zeta} \leq \frac{n^{1-\zeta}}{1-\zeta}$, there exists some constant $C_2$ such that
\[
\begin{aligned} 
\frac{1}{n}\e{R_{n,1}} 
&\leq \frac{1}{n}\left( 4n\nu_1\rho^{{H}}  + \frac{Cn^{1-\zeta}}{1-\zeta} \right)\\
&\leq 4\nu_1\rho^{{H}} + \frac{C}{(1-\zeta)n^{\zeta}}\\
&\leq\frac{C_2}{n^{\zeta}},
\end{aligned} 
\]
where the last step is by the fact that $\rho^{H} = n^\lambda$ and that $\zeta \leq -\lambda$. 

Finally, we analyze the regret of $R_{n,3}$. Let $\tilde{n}_3 = \abs{\mc{L}_3}$. For any node $(h,i)\in\mc{L}_3$, since the parent of any $(h,i)\in J_h$ is in $I_{h-1}$, we know by Lemma~\ref{lemma:3} that the domain of $(h,i)$ is in $\mc{X}_{4\nu_1 \rho^{h-1}}$. Further, $(h,i)$ is not $2\nu_1\rho^h$-optimal by the definition of $J_h$. We then have
\[
\begin{aligned}
\e{R_{n,3}} &= \e{\sum_{t=1}^n (f^* - Y_t) \indicator_{\{(H_t,I_t) \in\mc{L}_3\}}}\\
&=\e{\sum_{t=1}^n (f^* - f(X_t)) \indicator_{\{(H_t,I_t)\in\mc{L}_3\}}} + \e{\sum_{t=1}^n (f(X_t) - Y_t) \indicator_{\{(H_t,I_t) \in\mc{L}_3\}}}\\
&\leq \sum_{h=1}^H 4\nu_1 \rho^{h-1}\sum_{i:(h,i)\in J_h}\e{T_{h,i}(n)} +   \frac{C}{(1-\zeta)n^{\zeta-1}}\\
&\leq \sum_{h=1}^H 4\nu_1 \rho^{h-1} \abs{J_h}\left[\left(\frac{2n^{\alpha/\xi}}{\nu_1\rho^{h}}\right)^{\frac{1}{1-\eta}} + 2 + \frac{1}{\alpha -3}\right] + \frac{C}{(1-\zeta)n^{\zeta-1}}
\end{aligned}
\]
where the last step is by an application of Lemma~\ref{lemma:expectedt}. Further, since the parent of $J_h$ is in $I_{h-1}$, we know from Lemma~\ref{lemma:I_h} that $\abs{J_h}\leq 2\abs{I_{h-1}} \leq 2 C_3\left(\nu_{2} \rho^{h-1}\right)^{-d^{\prime}}$ for some constant $C_3$. Therefore, there exists some constant $C_4$, such that
\[
\begin{aligned}
\frac{1}{n}\e{R_{n,3}}&\leq \frac{1}{n}\sum_{h=1}^H 8C_3\nu_1 \rho^{h-1} \left(\nu_{2} \rho^{h-1}\right)^{-d^{\prime}}\left[\left(\frac{2n^{\alpha/\xi}}{\nu_1\rho^{h}}\right)^{\frac{1}{1-\eta}} + 2 + \frac{1}{\alpha -3}\right]   + \frac{C}{(1-\zeta)n^{\zeta}}\leq   \frac{C_4}{n^{\zeta}},
\end{aligned}
\]
where the last step holds because $\frac{1}{n}\sum_{h=1}^H 8C_3\nu_1\rho^{h-1} \left(\nu_{2} \rho^{h-1}\right)^{-d^{\prime}}\left(\frac{2n^{\alpha/\xi}}{\nu_1\rho^{h}}\right)^{\frac{1}{1-\eta}}$ is in the order of $O(n^{\lambda})$, and by the fact that $\zeta\leq - \lambda$. 

Putting everything together, we arrive at the desired convergence result:
\[
\abs{f^* - \frac{1}{n}\e{\sum_{t=1}^n Y_t}} =  \abs{\frac{1}{n}\e{R_n}} =  \abs{\frac{1}{n}\e{R_{\mc{T}} + R_{n,1} + R_{n,3}}}\leq  \frac{C_0}{n^{\zeta}},
\]
where $C_0>0$ is a proper constant that can be calculated from $C,R, \alpha, \nu_1,\bar{H}$ and $\zeta$.

\clearpage
\section{Proof of Theorem~\ref{thm:mcts}}\label{appendix:mcts}
In the following, we provide a complete proof for Theorem~\ref{thm:mcts}. The idea of this proof is built upon the analysis of fixed-depth Monte-Carlo tree search derived in~\cite{shah2019reinforcement}. Given the value function oracle $\hat{V}$ at the leaf nodes, a depth-$D$ MCTS can be approximately considered as $D$ steps of value iteration starting from $\hat{V}$. Let $V^{(d)}$ be the value function after $d$ steps of exact value iteration with $V^{(0)} = \hat{V}$. Since value iteration is a contraction mapping with respect to the $L^\infty$ norm, we have $\left\|V^{(d+1)}-V^{*}\right\|_{\infty} \leq \gamma\left\|V^{(d)}-V^{*}\right\|_{\infty}$, where $V^*$ is the optimal value function. Therefore, we conclude that 
\begin{equation}\label{eqn:vi} 
\left|V^{(D)}(s^{(0)})-V^{*}(s^{(0)})\right| \leq \gamma^{D}\left\|\hat{V}-V^{*}\right\|_{\infty} = \gamma^D \varepsilon_0
\end{equation} 
for the MCTS root node $s^{(0)}$. 

In the following, we will show that the empirical average reward collected at the root node of MCTS (denoted as $\tilde{v}^{(0)}(s^{(0)})/n$ in Algorithm~\ref{alg:POLY-HOOT}) is within $O(n^{\eta-1})$ of $V^{(D)}(s^{(0)})$ after $n$ rounds of MCTS simulations. The proof is based on an inductive procedure that we will go through in the following sections. Before that, we first introduce a lemma that will be useful throughout. 

\begin{lemma}\label{lemma:x+y}
	Consider real-valued random variables $X_i,Y_i$ for $i\geq 1$, where $X_i$'s are independent and identically distributed, taking values in $[-B,B]$ for some $B>0$. $Y_i$'s are independent of $X_i$'s, satisfying the following two properties:
	
	\quad A. Convergence: Let $\bar{Y}_{n}=\frac{1}{n}\left(\sum_{i=1}^{n} Y_{i}\right)$; then there exists $C>0, 0<\zeta\leq1/2 $, and $\mu_Y$, such that for every integer $n\geq 1$
	\begin{equation}
	\abs{\e{\bar{Y}_n} - \mu_Y} \leq  \frac{C}{n^{\zeta}}
	\end{equation}
	\quad B. Concentration: There exist constants $\beta > 1, \xi>0,$ and $1/2\leq \eta < 1$, such that for every $z\geq 1$ and every integer $n\geq 1$:
	\begin{equation}
	\mathbb{P}\left(n \bar{Y}_{n}-n \mu_{Y} \geq n^{\eta} z  \right) \leq \frac{\beta}{z^{\xi}}, \quad \mathbb{P}\left(n \bar{Y}_{n}-n \mu_{Y} \leq-n^{\eta} z\right) \leq \frac{\beta}{z^{\xi}}.
	\end{equation}
	
	Let $Z_i = X_i +\gamma Y_i$ for some $0 < \gamma <1$, and let $\bar{Z}_{n}=\frac{1}{n} \sum_{i=1}^{n} Z_{i}=\frac{1}{n}\sum_{i=1}^{n}\left(X_{i}+ \gamma Y_{i}\right)$. Define $\mu_X = \e{X_1}$. Then, the following properties are satisfied:
	
	\quad A. Convergence: 
	\begin{equation}
	\abs{\e{\bar{Z}_n} -(\mu_x + \gamma \mu_Y)} \leq  \frac{C}{n^{\zeta}}
	\end{equation}
	\quad B. Concentration: There exists a constant $\beta' > 1$ depending on $\gamma, \xi,\beta$ and $B$, such that for every $z\geq 1$ and every integer $n\geq 1$:
	\[
	\begin{aligned}
	&\mathbb{P}\left(n \bar{Z}_{n}-n (\mu_X + \gamma \mu_{Y}) \geq n^{\eta} z \right) \leq \frac{\beta'}{z^{\xi}},\\
	&\mathbb{P}\left(n  \bar{Z}_{n}-n (\mu_X + \gamma \mu_{Y})\leq-n^{\eta} z\right) \leq \frac{\beta'}{z^{\xi}}.
	\end{aligned}
	\]
\end{lemma}
\begin{proof}
	We first prove the convergence property of $\bar{Z}_n$.  $\abs{\e{\bar{Z}_n} - (\mu_X+\gamma \mu_Y)} =\abs{\gamma \e{\bar{Y}_n} - \gamma \mu_Y} \leq  \frac{\gamma C}{n^{\zeta}}  \leq \frac{C}{n^{\zeta}}$. 
	
	We then prove the concentration property of $\bar{Z}_n$. Let $\bar{X}_n = \frac{1}{n}\sum_{i=1}^n X_i$. By Hoeffding's inequality, we know $\prob{\bar{X}_n - \mu_X \geq \varepsilon}\leq \exp(\frac{-2n\varepsilon^2}{B^2})$. Then,
	\[
	\begin{aligned}
	&\pp\left(n \bar{Z}_{n}-n (\mu_X + \gamma \mu_{Y}) \geq n^{\eta} z\right)\\
	=&\pp\left(n\bar{X}_{n} - n\mu_X+n\gamma  \bar{Y}_{n}-n \gamma \mu_{Y} \geq n^{n} z \right)\\
	\leq&  \pp\left(n\bar{X}_{n} - n\mu_X\geq \frac{n^{\eta} z}{2}\right)+\pp\left(n\bar{Y}_{n}-n \mu_{Y} \geq \frac{n^{\eta} z}{2 \gamma } \right)\\
	\leq& \exp\left(-\frac{n^{2\eta - 1}z^2}{2B^2}\right) + \frac{2^\xi\beta  \gamma^\xi}{z^\xi}\\
	\leq& \frac{\beta'}{z^\xi}
	\end{aligned}
	\]
	where $\beta'$ is a constant large enough depending on $\gamma, \xi,\beta$ and $B$. The other side of the concentration inequality follows similarly. 
\end{proof}
\subsection{Base case}
We wanted to inductively show that the empirical mean reward collected at the root node of MCTS is within $O(n^{\eta-1})$ of the value iteration result $V^{(D)}(s^{(0)})$ after $n$ rounds of MCTS simulations. We start with the induction base case at MCTS depth $D-1$, which contains the parent nodes of the leaf nodes at level $D$. 

First, notice that there are only finitely many nodes at MCTS depth $D-1$ when $n$ goes to infinity, even though both the state space and the action space are continuous. This is because the HOO tree has limited depth at each MCTS node, and we repeatedly take the same action at a leaf of the HOO tree, resulting in a finite number of actions tried at each state. Further, we have assumed deterministic transitions, and thus each action at a given state repeatedly leads to the same destination state throughout the MCTS process. Combining those two properties gives finite number of nodes in the MCTS tree. 

Consider a node denoted as $i$ at depth $D-1$, and let $s_{i,D-1}$ denote the corresponding state. According to the definition of Algorithm~\ref{alg:POLY-HOOT}, whenever state $s_{i,D-1}$ is visited, the bandit algorithm will select an action $a$ from the action space, and the environment will transit to state $s_{D}'=s_{i,D-1}\circ a$ at depth $D$. The corresponding reward collected at node $i$ of depth $D-1$ would be $R(s_{i,D-1},a) + \gamma \tilde{v}^{(D)}(s_D')$, where the reward $R(s,a)$ is an independent random variable taking values bounded in $[-R_{max}, R_{max}]$. Recall that we use a deterministic value function oracle at depth $D$, and hence $\tilde{v}^{(D)}(s_D') = \hat{V}(s_D')$ is fully determined once the action $a$ is known. We also know the reward is bounded in $[-\frac{R_{max}}{1-\gamma}-\epsilon_0, \frac{R_{max}}{1-\gamma}+\epsilon_0]$, where $\epsilon_0$ is the largest possible mistake made by the value function oracle. We can then apply Lemma~\ref{lemma:x+y} here, with the $X$'s in Lemma~\ref{lemma:x+y} corresponding to the partial sums of independent rewards $R(s_{i,D-1}, a)$, the $Y$'s corresponding to the deterministic values $\tilde{v}^{(D)}(s_D')$. From the result of Lemma~\ref{lemma:x+y}, we know for the given $\alpha^{(D-1)},\eta^{(D-1)}$ and $\xi^{(D-1)}$ calculated from~\eqref{eqn:parameters}, there exists a constant $\beta^{(D-1)}$ such that the rewards collected at $s_{i,D-1}$ satisfy the concentration property~\eqref{eqn:concentration} required by Theorem~\ref{thm:bandit}. 

Further, let $f_n$ in Theorem~\ref{thm:bandit} be the mean-payoff function when state $s_{i,D-1}$ is visited for the $n$-th time, i.e., $f_n(a) = \e{R(s_{i,D-1},a)} + \gamma \hat{V}(s_D')$. Then since the rewards are stationary, there apparently exists a function $f = f_n,\ \forall n\geq 1$ such that the convergence~\eqref{eqn:convergence} property is satisfied with arbitrary value of $\zeta$ such that $0 < \zeta < 1-\frac{\alpha}{\xi(1-\eta)}$. Since we use exactly the same Algorithms~\ref{alg:HOO_query} and~\ref{alg:HOO_update} in the MCTS simulations as the ones stated in Theorem~\ref{thm:bandit}, the results of Theorem~\ref{thm:bandit} apply. 

Finally, define 
$$
\mu_*^{(D-1)}(s_{i,D-1}) = \sup_{a\in A}\left\{\e{R(s_{i,D-1},a)} + \gamma \tilde{v}^{(D)}(s_{i,D-1}\circ a)\right\}.
$$ 
Applying Theorem~\ref{thm:bandit} gives the following result:
\begin{proposition}
	Consider a node $i$ at depth $D-1$ of MCTS with the corresponding state $s_{i, D-1}$. Let $\tilde{v}_n^{(D-1)}(s_{i,D-1})$ denote the value of $\tilde{v}^{(D-1)}(s_{i,D-1})$ at the end of the $n$-th round of MCTS simulations. Then, for a given $\xi^{(D-1)}>0, \eta^{(D-1)}\in [\frac{1}{2},1), \alpha^{(D-1)} > 3$, and a proper value of $\beta^{(D-1)}$ given by Lemma~\ref{lemma:x+y}, we have
		
	\quad A. Convergence: There exists some constant $C_0 > 0$ and $0 < \zeta^{(D-1)} < 1-\frac{\alpha^{(D-1)}}{\xi^{(D-1)}(1-\eta^{(D-1)})}$, such that 
	\[
	\abs{\frac{1}{n}\e{\tilde{v}_n^{(D-1)}(s_{i,D-1})-\mu_*^{(D-1)}(s_{i,D-1})}} \leq  \frac{C_0}{n^{\zeta^{(D-1)}}}.
	\]
	
	\quad B. Concentration: There exist constants $\beta' > 1, \xi'>0,$ and $1/2\leq \eta' < 1$, such that for every $z\geq 1$ and every integer $n\geq 1$:
	\[
	\begin{aligned}
	&\pp\left(\tilde{v}_n^{(D-1)}(s_{i,D-1}) - n \mu_*^{(D-1)}(s_{i,D-1}) \geq n^{\eta'}z   \right) \leq \frac{\beta'}{z^{\xi'}},\\
	&\prob{\tilde{v}_n^{(D-1)}(s_{i,D-1}) - n \mu_*^{(D-1)}(s_{i,D-1}) \leq -n^{\eta'}z} \leq \frac{\beta'}{z^{\xi'}},
	\end{aligned}
	\]
	where $\eta' = \frac{\frac{\alpha^{(D-1)}}{\xi^{(D-1)} (1-\eta^{(D-1)})}+d'+\frac{1}{1-\eta^{(D-1)}}}{1+d'+\frac{1}{1-\eta^{(D-1)}}}$ with constant $d'$ defined in Definition~\ref{dfn:d'}, $\xi' = (\alpha^{(D-1)}-3)/2$, and $\beta'>1$ depends on $\alpha^{(D-1)},\beta^{(D-1)},\eta^{(D-1)},\xi^{(D-1)}$ and $\bar{H}$.
\end{proposition} 
Since $\alpha^{(D-1)} < \xi^{(D-1)}(1-\eta^{(D-1)})$, we can see $0<\eta'<1$. We would also like to remark that the definition of $\mu_*^{D-1}(s_{i,D-1})$ is exactly the value function estimation at $s_{i,D-1}$ after one step of value iteration starting from $\hat{V}$. If we set $\alpha^{(D-1)} = \xi^{(D-1)}\eta^{(D-1)}(1-\eta^{(D-1)})$, then $\zeta^{(D-1)}\in (0, \frac{1}{2})$.This completes the base case for our induction. 

\subsection{Induction step}
We have shown that the convergence and concentration requirements are satisfied from depth $D$ to depth $D-1$. In the following, we will recursively show that these properties also hold from depth $d$ to depth $d-1$ for all $1 \leq d \leq D-1$.

Consider a node denoted as $i$ at depth $d-1$, and let $s_{i,d-1}$ denote the corresponding state. Again, according to the definition of Algorithm~\ref{alg:POLY-HOOT}, whenever state $s_{i,d-1}$ is visited, the bandit algorithm will select an action $a$ from the action space, and the environment will transit to state $s_{d}'=s_{i,d-1}\circ a$ at depth $d$. The corresponding reward collected at node $i$ of depth $d-1$ would be $R(s_{i,d-1},a) + \gamma \tilde{v}^{(d)}(s_d')$, where the reward $R(s,a)$ is an independent random variable taking values bounded in $[-R_{max}, R_{max}]$. Our induction hypothesis assumes that $\tilde{v}^{(d)}$ satisfies the convergence and concentration properties for all states at depth $d$, with parameters $\alpha^{(d)}, \xi^{(d)},\eta^{(d)}$ defined by~\eqref{eqn:parameters} and proper value of $\beta^{(d)}$. Therefore, we can again apply Lemma~\ref{lemma:x+y} here, with the $X$'s in Lemma~\ref{lemma:x+y} corresponding to the partial sums of independent rewards $R(s_{i,d-1}, a)$, and the $Y$'s corresponding to $\tilde{v}^{(d)}(s_d')$ that satisfy the convergence and concentration properties by our induction hypothesis. From the result of Lemma~\ref{lemma:x+y}, we know for the given $\alpha^{(d-1)},\eta^{(d-1)}$ and $\xi^{(d-1)}$ calculated from~\eqref{eqn:parameters}, there exists a constant $\beta^{(d-1)}$ such that the rewards collected at $s_{i,d-1}$ satisfy the concentration property~\eqref{eqn:concentration} required by Theorem~\ref{thm:bandit}. 

Let $f_n$ in Theorem~\ref{thm:bandit} be the mean-payoff function after state $s_{i,D-1}$ is visited for the $n$-th time, i.e., $f_n(a) = \e{R(s_{i,D-1},a)} + \gamma \tilde{v}^{(d)}_{n}(s_d')/n$. Define $f(a) = \e{R(s_{i,D-1},a)} + \gamma \mu_*^{(d)}(s_d')$, then we can see the convergence requirement~\eqref{eqn:convergence} is also satisfied by $f_n$ and $f$, with $\zeta = \zeta^{(d)}$. Therefore, the results of Theorem~\ref{thm:bandit} apply.

Finally, define 
$$
\mu_*^{(d-1)}(s_{i,d-1}) = \sup_{a\in A}\left\{\e{R(s_{i,d-1},a)} + \gamma \mu_*^{(d)}(s_{i,d-1}\circ a)\right\}.
$$ 
A direct application of Theorem~\ref{thm:bandit} gives the following result:
\begin{proposition}
	For a node $i$ at depth $d-1$ of MCTS with the corresponding state $s_{i, d-1}$. Let $\tilde{v}_n^{(d-1)}(s_{i,d-1})$ denote the value of $\tilde{v}^{(d-1)}(s_{i,d-1})$ at the end of the $n$-th round of MCTS simulations. Then, for a given $\xi^{(d-1)}>0, \eta^{(d-1)}\in [\frac{1}{2},1), \alpha^{(d-1)} > 3$, and a proper value of $\beta^{(d-1)}$ given by Lemma~\ref{lemma:x+y}, we have
	
	\quad A. Convergence: There exists some constant $C_0 > 0$ and $0 < \zeta^{(d-1)} < 1-\frac{\alpha^{(d-1)}}{\xi^{(d-1)}(1-\eta^{(d-1)})}$, such that 
	\begin{equation}\label{eqn:induction}
	\abs{\frac{1}{n}\e{\tilde{v}_n^{(d-1)}(s_{i,d-1})-\mu_*^{(d-1)}(s_{i,d-1})}} \leq  \frac{C_0}{n^{\zeta^{(d-1)}}}.
	\end{equation}
	
	\quad B. Concentration: There exist constants $\beta' > 1, \xi'>0,$ and $1/2\leq \eta' < 1$, such that for every $z\geq 1$ and every integer $n\geq 1$:
	\[
	\begin{aligned}
	&\pp\left(\tilde{v}_n^{(d-1)}(s_{i,d-1}) - n \mu_*^{(d-1)}(s_{i,d-1}) \geq n^{\eta'}z  \right) \leq \frac{\beta'}{z^{\xi'}},\\
	&\prob{\tilde{v}_n^{(d-1)}(s_{i,d-1}) - n \mu_*^{(d-1)}(s_{i,d-1}) \leq -n^{\eta'}z } \leq \frac{\beta'}{z^{\xi'}},
	\end{aligned}
	\]
	where $\eta' = \frac{\frac{\alpha^{(d-1)}}{\xi^{(d-1)} (1-\eta^{(d-1)})}+d'+\frac{1}{1-\eta^{(d-1)}}}{1+d'+\frac{1}{1-\eta^{(d-1)}}}$ with constant $d'$ defined in Definition~\ref{dfn:d'}, $\xi' = (\alpha^{(d-1)}-3)/2$, and $\beta'>1$ depends on $\alpha^{(d-1)},\beta^{(d-1)},\eta^{(d-1)},\xi^{(d-1)}$ and $\bar{H}$.
\end{proposition} 
Since $\alpha^{(d-1)} < \xi^{(d-1)}(1-\eta^{(d-1)})$, we can see that $0<\eta'<1$. If we set $\alpha^{(d-1)} = \xi^{(d-1)}\eta^{(d-1)}(1-\eta^{(d-1)})$, then $\zeta^{(d-1)}\in (0, \frac{1}{2})$. Notice that the definition of $\mu_*^{d-1}(s_{i,d-1})$ is exactly the value function estimation at $s_{i,d-1}$ after $D-d$ steps of value iteration starting from $\hat{V}$. This completes the proof of the induction step. 

\subsection{Completing proof of Theorem~\ref{thm:mcts}}
Following an inductive procedure, we can see that the convergence result~\eqref{eqn:induction} also holds at the MCTS root node $s^{(0)}$. After $n$ rounds of MCTS simulations starting from the root node, the empirical mean reward collected at $s^{(0)}$ satisfies:
\begin{equation}
\abs{\frac{1}{n}\e{\tilde{v}_n^{(0)}(s^{(0)})-\mu_*^{(0)}(s^{(0)})}} \leq   \frac{C_0}{n^{\zeta^{(0)}}},
\end{equation}
where $\mu_*^{(0)}(s^{(0)})$ is the value function estimation for $s^{(0)}$ after $D$ rounds of value iteration starting from $\hat{V}$, and $\zeta^{(0)}\in (0,\frac{1}{2})$ if we set $\alpha^{(0)} = \xi^{(0)}\eta^{(0)}(1-\eta^{(0)})$. Recall from Equation~\eqref{eqn:vi} that $\left|\mu_*^{(0)}(s^{(0)})-V^{*}(s^{(0)})\right| \leq \gamma^{D}\left\|\hat{V}-V^{*}\right\|_{\infty} = \gamma^D \varepsilon_0$. By the triangle inequality, we conclude that
\[
\abs{\frac{1}{n}\e{\tilde{v}_n^{(0)}(s^{(0)})-V^{*}(s^{(0)})}} \leq  O\left(\frac{1}{n^{\zeta}}\right) +  \gamma^D \varepsilon_0,
\] 
for some $0 < \zeta < 1/2$. This completes the proof of Theorem~\ref{thm:mcts}.

\section{Technical Lemmas}\label{appendix:lemmas}

\begin{lemma}\label{lemma:3}
	(Lemma 3 in~\cite{bubeck2011x}) Under Assumptions~\ref{assumption:1} and~\ref{assumption:2}, for some region $\mc{P}_{h,i}$, if $\Delta_{h,i}\leq c\nu_1 \rho^h$ for some constant $c\geq 0$, then all the arms in $\mc{P}_{h,i}$ are $\max\{2c,c+1\}$-optimal.
\end{lemma}
\begin{proof}
	This lemma is stated in exactly the same as way Lemma 3 in~\cite{bubeck2011x}, and we therefore omit the proof here. 
\end{proof}

\begin{lemma}\label{lemma:I_h}
	There exists some constant $C>0$, such that $\abs{I_h}\leq C(\nu_2\rho^h)^{-d'}$ for all $h\geq 0$.
\end{lemma}
\begin{proof}
	This result is the same as the second step in the proof of Theorem 6 in~\cite{bubeck2011x}. We therefore omit the proof here.
\end{proof}

\begin{lemma}\label{lemma:pu<f}
	Let Assumptions~\ref{assumption:1} and~\ref{assumption:2} hold. Then for every optimal node~\footnote{Recall Definition~\ref{dfn:optimalnode}.} $(h,i)$ and any integer $n\geq 1$, there exists a constant $\beta_1 > 1$, such that
	\[
	\prob{U_{h,i}(n)\leq f^*}\leq \frac{\beta_1}{n^{\alpha-1}}.
	\]
\end{lemma}
\begin{proof}
	If $(h,i)$ is not played during the first $n$ rounds, then by assumption $U_{h,i}(n) = \infty$ and the inequality trivially holds. Now we focus on the case where $T_{h,i}(n)\geq 1$. From Lemma~\ref{lemma:3}, we know that $f^*- f(x)\leq \nu_1 \rho^h,\ \forall x\in \mc{P}_{h,i}$. Then we have $\sum_{t=1}^{n}\left(f\left(X_{t}\right)+\nu_{1} \rho^{h}-f^{*}\right) \mathbb{I}_{\left\{\left(H_{t}, I_{t}\right) \in \mathcal{C}(h, i)\right\}} \geq  0$. Therefore, 
	\[
	\begin{aligned}
	&\mathbb{P}\left(U_{h, i}(n) \leq  f^{*} \ \text { and } \ T_{h, i}(n) \geq 1\right)\\
	=&\mathbb{P}\left(\widehat{\mu}_{h, i}(n)+ n^{\alpha/\xi}T_{h,i}(n)^{\eta-1} +\nu_{1} \rho^{h} \leq f^{*} \  \text { and } \  T_{h, i}(n) \geq  1\right)\\
	=&\mathbb{P}\left(T_{h, i}(n) \widehat{\mu}_{h, i}(n)+T_{h, i}(n)\left(\nu_{1} \rho^{h}-f^{*}\right) \leq - n^{\alpha/\xi}T_{h,i}(n)^{\eta} \ \text { and }\ T_{h, i}(n) \geq 1\right)\\
	=&\mathbb{P}\Bigg(\sum_{t=1}^{n}\left(Y_{t}-f\left(X_{t}\right)\right) \mathbb{I}_{\left\{\left(H_{t}, I_{t}\right) \in \mathcal{C}(h, i)\right\}}+\sum_{t=1}^{n}\left(f\left(X_{t}\right)+\nu_{1} \rho^{h}-f^{*}\right) \mathbb{I}_{\left\{\left(H_{t}, I_{t}\right) \in \mathcal{C}(h, i)\right\}} \\
	&\qquad \leq - n^{\alpha/\xi}T_{h,i}(n)^{\eta} \ \text { and }\ T_{h, i}(n) \geq 1 \Bigg)\\
	\leq & \mathbb{P}\left(\sum_{t=1}^{n}\left(f\left(X_{t}\right)-Y_{t}\right) \mathbb{I}_{\left\{\left(H_{t}, I_{t}\right) \in \mathcal{C}(h, i)\right\}} \geq n^{\alpha/\xi}T_{h,i}(n)^{\eta}\  \text { and }\ T_{h, i}(n) \geq  1\right)
	\end{aligned}
	\]
	Since the HOO tree has limited depth, the total number of nodes played in $\mc{C}(h,i)$ is upper bounded by some constant $C>1$ that is independent of $n$. Let $X^j$ denote the $j$-th new node played in $\mc{C}(h,i)$, denote the number of times $X^j$ is played as $n_j$, and let $Y^j_t\ (1 \leq t\leq n_j)$ be the corresponding reward the $t$-th time arm $X^j$ is played. Then, by the union bound, we have
	\[
	\begin{aligned} &\mathbb{P}\left(\sum_{t=1}^{n}\left(f\left(X_{t}\right)-Y_{t}\right) \mathbb{I}_{\left\{\left(H_{t}, I_{t}\right) \in \mathcal{C}(h, i)\right\}} \geq n^{\alpha/\xi}T_{h,i}(n)^{\eta}\  \text { and }\ T_{h, i}(n) \geq  1\right)\\
	\leq &\sum_{T_{h,i}(n)=1}^n \mathbb{P} \left(\sum_{t=1}^{n} \left(f\left(X_{t}\right)-Y_{t}\right) \mathbb{I}_{\left\{\left(H_{t}, I_{t}\right) \in \mathcal{C}(h, i)\right\}} \geq n^{\alpha/\xi}T_{h,i}(n)^{\eta}\right)\\
	= &\sum_{T_{h,i}(n)=1}^n\mathbb{P} \left(\sum_{j=1}^{\bar{H}}\sum_{t=1}^{n_j}\left(f\left(X^j\right)-Y^j_{t}\right)  \geq n^{\alpha/\xi}T_{h,i}(n)^{\eta}\right)\\
	\leq & \sum_{T_{h,i}(n)=1}^n\sum_{j=1}^{C}\mathbb{P} \left(\sum_{t=1}^{n_j}\left(f\left(X^j\right)-Y^j_{t}\right)  \geq \frac{n^{\alpha/\xi}}{C}n_j^{\eta}\right)\\
	\leq & \frac{\beta_1}{n^{\alpha-1}},
	\end{aligned}
	\]
	where $\beta_1>1$ is a constant depending on $C$ and $\beta$, and in the last inequality we applied the concentration property of the bandit problem~\eqref{eqn:concentration}. Notice that we can only use the concentration property when the requirement $z = \frac{n^{\alpha/\xi}}{\bar{H}}\geq 1$ is satisfied, but when $z < 1$, the inequality also trivially holds because $\frac{\beta}{z^{\xi}}  > 1$. This completes the proof of $\prob{U_{h,i}(n)\leq f^*}\leq \frac{\beta_1}{n^{\alpha-1}}.$
\end{proof}

\begin{lemma}\label{lemma:expectedtpre}
	(Lemma 14 in~\cite{bubeck2011x}) Let $(h,i)$ be a suboptimal node. Let $0 \leq k \leq h-1$ be the largest depth such that $(k,i_k^*)$ is on the path from the root $(0,1)$ to $(h,i)$, i.e., $(k,i_k^*)$ is the lowest common ancestor (LCA) of $(h,i)$ and the optimal path. Then, for all integers $u\geq 0$, we have
	$$
	\begin{aligned} 
	\mathbb{E}\left[T_{h, i}(n)\right] \leq u+\sum_{t=u+1}^{n} \mathbb{P}\left\{\left[U_{s, i_{s}^{*}}(t) \leq  f^{*} \text { for some } s \in\{k+1, \ldots, t-1\}\right]\right.\\
	\text { or }\left.\left[T_{h, i}(t)>u \text { and } U_{h, i}(t)>f^{*}\right]\right\}.
	\end{aligned}
	$$
\end{lemma}
\begin{proof}
	This lemma is stated in exactly the same way as Lemma 14 in~\cite{bubeck2011x}, and the proof follows similarly. We hence omit the proof here. 
\end{proof}

\begin{lemma}\label{lemma:pu>f}
	For all integers $t \leq n$, for any suboptimal node $(h, i)$ such that $\Delta_{h, i}>\nu_{1} \rho^{h},$ and for all integers $u \geq A_{h,i}(n) = \ceil{\left(\frac{2n^{\alpha/\xi}}{\Delta_{h,i} - \nu_1\rho^h}\right)^{\frac{1}{1-\eta}}}$, there exists a constant $\beta_2 >1$, such that 
	$$
	\mathbb{P}\left(U_{h, i}(t)>f^{*} \text { and } T_{h, i}(t)>u\right)\leq \frac{\beta_2 t}{n^\alpha}.
	$$
\end{lemma}
\begin{proof}
	The proof idea follows almost the same procedure as the proof of Lemma 16 in~\cite{bubeck2011x}, and we repeat it here due to some minor differences. First, notice that the $u$ defined in the statement of the lemma satisfies $n^{\alpha/\xi} u^{\eta - 1} + \nu_1 \rho \leq \frac{\Delta_{h,i} + \nu_1 \rho^h}{2}$. Then we have
	\[
	\begin{aligned}
	&\mathbb{P}\left(U_{h, i}(t)>f^{*}\ \text { and }\ T_{h, i}(t)>u\right)\\
	=&\mathbb{P}\left(\widehat{\mu}_{h, i}(t)+n^{\alpha / \xi}u^{\eta - 1}+\nu_{1} \rho^{h}>f_{h, i}^{*}+\Delta_{h, i} \  \text { and } \ T_{h, i}(t)>u\right)\\
	\leq &\mathbb{P}\left(\widehat{\mu}_{h, i}(t)>f_{h, i}^{*}+\frac{\Delta_{h, i}-\nu_{1} \rho^{h}}{2} \text { and } T_{h, i}(t)>u\right)\\
	\leq & \mathbb{P}\left(T_{h, i}(t)\left(\widehat{\mu}_{h, i}(t)-f_{h, i}^{*}\right)>\frac{\Delta_{h, i}-\nu_{1} \rho^{h}}{2} T_{h, i}(t) \text { and } T_{h, i}(t)>u\right)\\
	\leq & \mathbb{P}\left(\sum_{s=1}^{t}\left(Y_{s}-f\left(X_{s}\right)\right) \mathbb{I}_{\left\{\left(H_{s}, I_{s}\right) \in \mathcal{C}(h, i)\right\}}>\frac{\Delta_{h, i}-\nu_{1} \rho^{h}}{2} T_{h, i}(t) \text { and } T_{h, i}(t)>u\right)\\
	\leq & \sum_{T_{h,i}(t)=u+1}^t\mathbb{P}\left(\sum_{s=1}^{t}\left(Y_{s}-f\left(X_{s}\right)\right) \mathbb{I}_{\left\{\left(H_{s}, I_{s}\right) \in \mathcal{C}(h, i)\right\}}>\frac{\Delta_{h, i}-\nu_{1} \rho^{h}}{2} T_{h, i}(t)\right),
	\end{aligned}
	\]
	where in the last step we used the union bound. Then, following a similar procedure as in the proof of Lemma~\ref{lemma:pu<f} (defining $X^j$ and $Y^j_t$, and then the concentration property), we get:
	\[
	\begin{aligned}
	& \sum_{T_{h,i}(t)=u+1}^t\mathbb{P}\left(\sum_{s=1}^{t}\left(Y_{s}-f\left(X_{s}\right)\right) \mathbb{I}_{\left\{\left(H_{s}, I_{s}\right) \in \mathcal{C}(h, i)\right\}}>\frac{\Delta_{h, i}-\nu_{1} \rho^{h}}{2} T_{h, i}(t)\right)\\
	\leq & \sum_{T_{h,i}(t)=u+1}^t \frac{\beta_2}{\left(\frac{\Delta _{h,i}-\nu_1\rho}{2}\right)^\xi \left(T_{h,i}(t)\right)^{\xi (1-\eta)}}\\
	\leq &\sum_{T_{h,i}(t)=u+1}^t \frac{\beta_2}{n^\alpha} \leq \frac{\beta_2 t}{n^\alpha},
	\end{aligned}
	\]
	where $\beta_2 > 1$ is a constant independent of $n$, and in the second step we used the fact that $T_{h,i}(t) > u \geq A_{h,i}(n) = \ceil{\left(\frac{2n^{\alpha/\xi}}{\Delta_{h,i} - \nu_1\rho^h}\right)^{\frac{1}{1-\eta}}}$. This completes our proof of $\mathbb{P}\left(U_{h, i}(t)>f^{*} \text { and } T_{h, i}(t)>u\right)\leq \frac{\beta_2 t}{n^\alpha}$.
\end{proof}

\begin{lemma}\label{lemma:expectedt}
	For any suboptimal node $(h,i)$ with $\Delta_{h,i} >\nu_1 \rho^h$ and any integer $n\geq 1$, there exist constants $\beta_1,\beta_2>1$, such that:
	\[
	\e{T_{h,i}(n)} \leq \left(\frac{2n^{\alpha/\xi}}{\Delta_{h,i} - \nu_1\rho^h}\right)^{\frac{1}{1-\eta}} + 1+\beta_1 + \frac{\beta_2}{\alpha -3}.
	\]
\end{lemma}
\begin{proof}
	Let $A_{h,i}(n) = \ceil{\left(\frac{2n^{\alpha/\xi}}{\Delta_{h,i} - \nu_1\rho^h}\right)^{\frac{1}{1-\eta}}}$. Then from Lemma~\ref{lemma:expectedtpre}, we know that
	\begin{small}
	\[
	\begin{aligned}
	\hspace{-10pt}\e{T_{h,i}(n)} \leq A_{h,i}(n)+\hspace{-12pt}\sum_{t=A_{h,i}(n)+1}^{n}\hspace{-5pt}\left(\mathbb{P}\left(T_{h, i}(t)>A_{h,i}(n) \text { and } U_{h, i}(t)>f^{*}\right)+\sum_{s=1}^{t-1} \mathbb{P}\left(U_{s, i_{s}^{*}}(t) \leq  f^{*}\right)\right)
	\end{aligned}
	\]
	\end{small}
	By replacing the right hand side with the results from Lemma~\ref{lemma:pu<f} and Lemma~\ref{lemma:pu>f}, we further have
	\[
	\begin{aligned}
	\e{T_{h,i}(n)} &\leq A_{h,i}(n) + \sum_{t=A_{h,i}(n) + 1}^n\left(\frac{\beta_2 t}{n^\alpha} + \sum_{s=1}^{t-1}\frac{\beta_1}{t^{\alpha-1}}\right)\\
	&\leq A_{h,i}(n)+ \frac{\beta_2}{n^{\alpha-2}} + \int_{u}^n \frac{\beta_1}{t^{\alpha-2}}dt\\
	&\leq \left(\frac{2n^{\alpha/\xi}}{\Delta_{h,i} - \nu_1\rho^h}\right)^{\frac{1}{1-\eta}} + 1+ \beta_2  + \frac{\beta_1}{\alpha-3}.
	\end{aligned}
	\]
	This completes our proof. 
\end{proof}

\begin{lemma}\label{lemma:pt>u}
	Let $(h,i)$ be a suboptimal node. Then for any $n\geq 1$ and any $u> A_{h,i}(n) = \ceil{\left(\frac{2n^{\alpha/\xi}}{\Delta_{h,i} - \nu_1\rho^h}\right)^{\frac{1}{1-\eta}}}$, there exist constants $\beta_1,\beta_2>1$, such that 
	$$
	\prob{T_{h,i}(n)> u } \leq \frac{\beta_2}{n^{\alpha-2}} + \frac{\beta_1(u-1)^{3-\alpha}}{\alpha - 3}.
	$$
\end{lemma}
\begin{proof}
	Clearly, this inequality holds for $n \leq u$, as $T_{h,i}(n)\leq n$ and the left hand side would be $0$ in this case. We therefore focus on the case $n > u$. 
	
	We first notice the following monotonicity of the $B$-values: according to the way that $B$-values are defined, the $B$-value of the descendants of a node $(h,i)$ would always be no smaller than the $B$-value of $(h,i)$ itself. Therefore, $B$-values do not decrease along a path from the root to a leaf. 
	
	Now, let $0 \leq k \leq h-1$ be the largest depth such that $(k,i_k^*)$ is on the path from the root $(0,1)$ to $(h,i)$. We define two events: $E_1 = \{\text{For each } t \in [u,n], B_{h,i}(t)\leq f^* \text{ or } T_{h,i}(t)\leq A_{h,i}(t) < u\}$, and $E_2 = \{\text{For each } t\in [u,n], B_{k+1,i^*_{k+1}}(t) > f^*\}$. We can verify that $E_1 \cap E_2 \subseteq \{T_{h,i}(n)\leq u\}$. To see this, suppose that for some $t\in [u,n]$ we have $B_{h,i}(t) \leq f^*$ and $B_{k+1,i^*_{k+1}}(t) > f^*$; then we know that we would not enter the node $(h,i)$. This is because by the monotonicity of the $B$-values, the ancestor of $(h,i)$ at level $k+1$ has a $B$-value no larger than $B_{h,i}(t)$, which in turn satisfies $B_{h,i}(t) \leq f^* < B_{k+1,i^*_{k+1}}(t)$. Therefore, we would always enter $B_{k+1,i^*_{k+1}}$ rather than the ancestor of $(h,i)$ at level $k+1$. In this case, $T_{h,i}$ would not increase at round $t$. Now consider the other case: suppose that for some $t\in [u,n]$ we have $T_{h,i}(t)\leq A_{h,i}(t) < u$ and $B_{k+1,i^*_{k+1}}(t) > f^*$. In this case, we could indeed possibly enter node $(h,i)$ and increase $T_{h,i}$ by $1$, but since $T_{h,i}(t)<u$, we still have $T_{h,i}(t+1)\leq u$ after increasing by $1$. Considering these two cases inductively, we can see that if $E_1\cap E_2$ holds, then $T_{h,i}(u-1) < u$ implies $T_{h,i}(n) \leq u$. Since $T_{h,i}(u-1) < u$ trivially holds, we can conclude that $E_1 \cap E_2 \subseteq \{T_{h,i}(n)\leq u\}$. 
	
	After we have $E_1 \cap E_2 \subseteq \{T_{h,i}(n)\leq u\}$, we know that $\{T_{h,i}(n)> u\} \subseteq E_1^c \cup E_2^c$, where $E^c$ denotes the complement of event $E$. This in turn gives us $\prob{\{T_{h,i}(n)> u\}} \leq  \prob{E_1^c} + \prob{E_2^c}$. From the definition of the $B$-values, $\left\{B_{k+1, i_{k+1}^*}(t) \leq  f^{*}\right\} \subset\left\{U_{k+1, i_{k+1}^{*}}(t) \leq  f^{*}\right\} \cup\left\{B_{k+2, i_{k+2}^{*}}(t) \leq  f^{*}\right\}$, and this can be applied recursively up to depth $t$, where the nodes in depth $t$ have not been played at round $t$ and satisfy $B_{t,i^*_t} = \infty > f^*$. Together with the fact that $U_{h,i}(t) \geq B_{h,i}(t)$ (by definition), we have
	\[
	\begin{aligned}
	&\prob{T_{h,i}(n)> u} \\
	\leq &\prob{\exists t\in [u,n], B_{h,i}(t) > f^* \text{ and } T_{h,i}(t) > A_{h,i}(t)} + \prob{\exists t \in [u,n], B_{k+1,i^*_{k+1}}(t) \leq f^*}\\
	\leq &\prob{\exists t\in [u,n], U_{h,i}(t) > f^* \text{ and } T_{h,i}(t) > A_{h,i}(t)}\\
	&\quad + \prob{\exists t \in [u,n], U_{k+1,i^*_{k+1}}(t) \leq  f^* \text{ or } U_{k+2,i^*_{k+2}}(t) \leq  f^* \text{ or }\dots \text{ or } U_{t-1,i^*_{t-1}}(t)\leq f^* }\\
	\leq &\sum_{t=u}^n\prob{U_{h,i}(t) > f^* \text{ and } T_{h,i}(t) > A_{h,i}(t)}\\
	&\quad + \sum_{t=u}^n\prob{U_{k+1,i^*_{k+1}}(t) \leq  f^* \text{ or } U_{k+2,i^*_{k+2}}(t) \leq  f^* \text{ or }\dots \text{ or } U_{t-1,i^*_{t-1}}(t)\leq f^* }\\
	\leq &\sum_{t=u}^n\prob{U_{h,i}(t) > f^* \text{ and } T_{h,i}(t) > A_{h,i}(t)} + \sum_{t=u}^n\sum_{s=1}^{t-1}\prob{U_{s,i^*_{s}}(t) \leq  f^*},
	\end{aligned}
	\]
	where in the last two steps we used the union bound. Since we know $\prob{U_{s, i_s^*}(t)\leq f^*}\leq \frac{\beta_1}{n^{\alpha-1}}$ from Lemma~\ref{lemma:pu<f}, and $\prob{U_{h,i}(t) > f^* \text{ and } T_{h,i}(t) > A_{h,i}(t)}\leq \frac{\beta_2t}{n^\alpha}$ from Lemma~\ref{lemma:pu>f}, we conclude that 
	\[
	\begin{aligned}
	&\sum_{t=u}^n\prob{U_{h,i}(t) > f^* \text{ and } T_{h,i}(t) > A_{h,i}(t)} + \sum_{t=u}^n\sum_{s=1}^{t-1}\prob{U_{s,i^*_{s}}(t) \leq  f^*}\\
	\leq &\sum_{t=u}^n \frac{\beta_2t}{n^\alpha} + \sum_{t=u}^n \sum_{s=1}^{t-1}\frac{\beta_1}{t^{\alpha-1}}\leq  \sum_{t=u}^n \frac{\beta_2n}{n^\alpha} + \beta_1\int_{u-1}^{\infty}t^{2-\alpha} dt\\
	\leq & \frac{\beta_2}{n^{\alpha-2}} + \frac{\beta_1(u-1)^{3-\alpha}}{\alpha-3}.
	\end{aligned}
	\]
	This completes the proof.
	
	We further remark that if $1<u \leq n$, then $\frac{1}{n^{\alpha-2}}\leq \frac{ u^{3-\alpha}n^{\alpha -3}}{n^{\alpha-2}} \leq \frac{(u-1)^{3-\alpha}}{n}$, which implies \begin{equation}\label{eqn:pt>u} 
	\prob{T_{h,i}(n)> u } \leq \frac{\beta_2(u-1)^{3-\alpha}}{n} + \frac{\beta_1(u-1)^{3-\alpha}}{\alpha - 3}.
	\end{equation} 
	Notice that this inequality also holds when $u > n$, because $T_{h,i}(n) \leq n < u$, and any non-negative value on the RHS is a trivial upper bound for $\prob{T_{h,i}(n)> u }$.
\end{proof}

\begin{remark}
	As a final remark, when we refer to the results of Lemmas~\ref{lemma:pu<f}, \ref{lemma:expectedtpre}, \ref{lemma:pu>f}, \ref{lemma:expectedt} and~\ref{lemma:pt>u}, we typically drop the constant factors $\beta_1$ and $\beta_2$ and proceed with $\beta_1=\beta_2=1$ instead. This does not affect our main results up to a constant factor. 
\end{remark}

\section{Details of the Simulations}\label{appendix:simulations}
In this section, we discuss details of the simulations and empirically evaluate the performance of \texttt{POLY-HOOT} on several classic control tasks. We have chosen three benchmark tasks from the OpenAI Gym~\citep{openai}, and extended them to the continuous-action settings as necessary. These tasks include CartPole, Inverted Pendulum Swing-up, and LunarLander. 

In the CartPole problem, a pole is attached to a cart through a joint. The task is to apply an appropriate horizontal force to the cart to prevent the pole from falling. For every time step that the pole remains standing (up to $15$ degrees from being vertical), a unit reward is given. We have also modified the CartPole problem to a more challenging setting with an increased gravity value (CartPole-IG) to better demonstrate the differences between the algorithms we compare. This new setting requires smoother actions, and bang-bang control strategies easily lead the pole to fall due to the increased momentum. The Inverted Pendulum Swing-up task is also a classic problem in control. A pendulum is attached to a frictionless pivot, starting from a random position. The task is to apply a force to the pendulum to swing it up and let it stay upright. At each time step, a reward is given based on the angle of the current position of the pendulum from being upright. In the LunarLander problem, the task is to design the control signals for a lunar lander to land smoothly on a landing pad. A negative reward is given every time the engine is fired, and a positive reward is given when the lander safely reaches the landing pad. 

In the original problem of CartPole, the action set is a discrete set $\{-1,1\}$. In our CartPole and CartPole-IG environments though, we have extended the action space to a continuous domain $[-1,1]$. In CartPole-IG, we have further increased the gravity value from $9.8$ to $50$, increased the mass of the pole from $0.1$ to $0.5$, and increased the length of the pole from $1$ to $2$. The other parameters have remained the same as the discrete setting in OpenAI Gym. For the task of Inverted Pendulum, we have manually reduced the randomness of the initial state to ensure that each run of the simulation is initialized more consistently. The reward discount factor was set to be $\gamma = 0.99$ for all the four tasks. The length of the horizon was taken as $T=150$. 

We compare the empirical performance of \texttt{POLY-HOOT} with three continuous MCTS algorithms, including UCT~\citep{kocsis2006bandit} with manually discretized actions, Polynomial Upper Confidence Trees (PUCT) with progressive widening~\citep{auger2013continuous}, and the original empirical implementation of HOOT~\citep{mansley2011sample} with a logarithmic bonus term. For all four algorithms, we have set the MCTS depth to be $D = 50$, except for the task of LunarLander where we set $D=100$ because this task takes a longer time to finish. We have set the number of simulations at each state to be $n=100$ rounds. For the UCT algorithm with discretized actions, we have fixed the number of actions to be $10$ and sampled the actions using a uniform grid. For PUCT with progressive widening, we have set the progressive widening coefficient to be $0.5$, i.e., the number of discrete action samples grows at a square-root order in time. For HOOT and $\texttt{POLY-HOOT}$, given the dimension $m$ of the action space, we have calculated the $\rho$ and $\nu_1$ parameters by $\rho = \frac{1}{4^{m}}$ and $\nu_1 = 4m$. For \texttt{POLY-HOOT}, we have set the maximum depth of the HOO tree covering to be $\bar{H} = 10$, and we have fixed $\alpha = 5, \xi = 20$, and $\eta = 0.5$. The value function oracle we have used is $\hat{V}(s) = 0,\forall s\in S$ for all four algorithms.

\begin{figure*}[!tbp]
	\centering
	\hspace{-0.2cm}\subfigure[]{\includegraphics[width=0.36\textwidth]{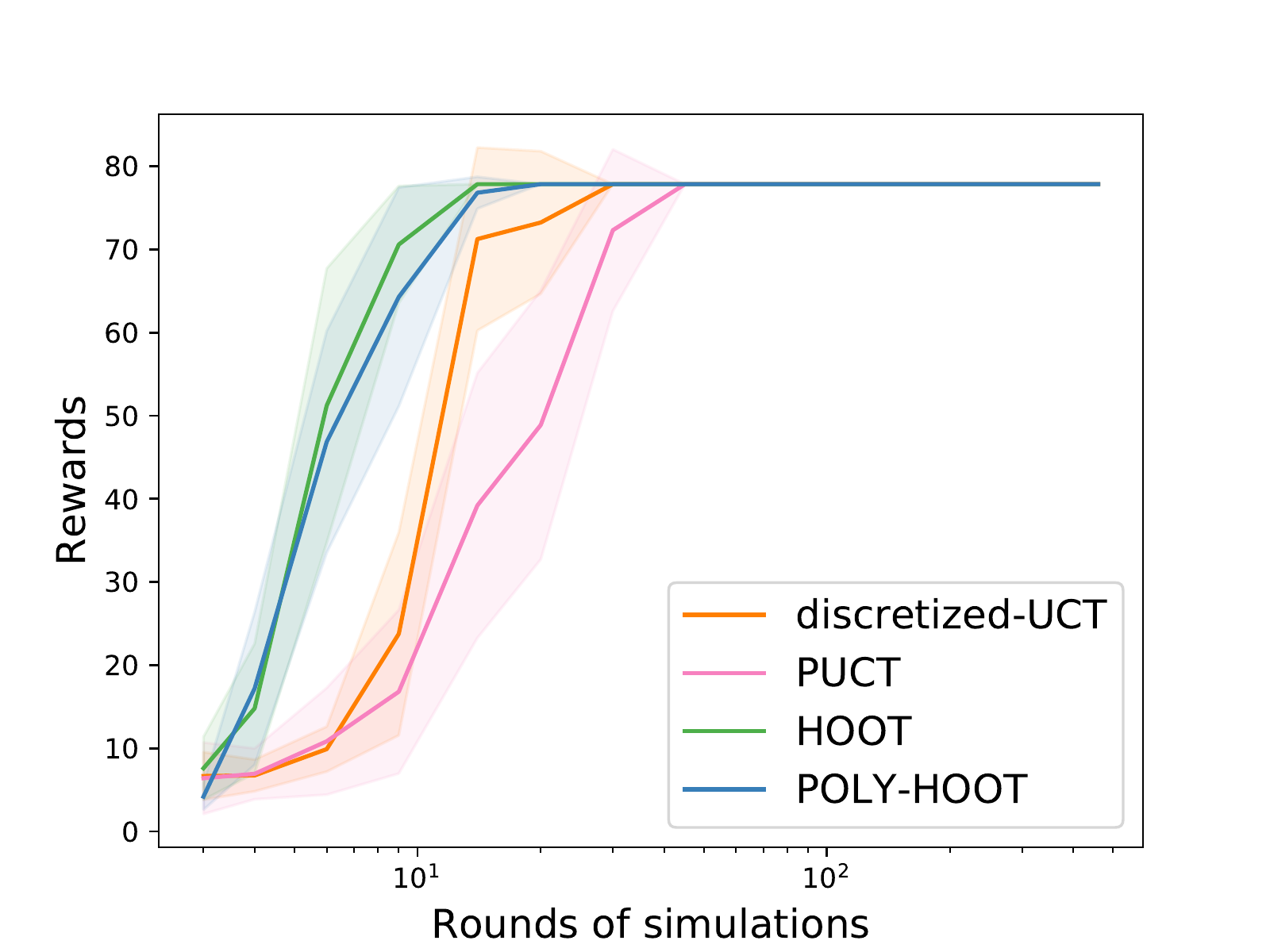}}
	\hspace{-0.55cm}\subfigure[]{\includegraphics[width=0.36\textwidth]{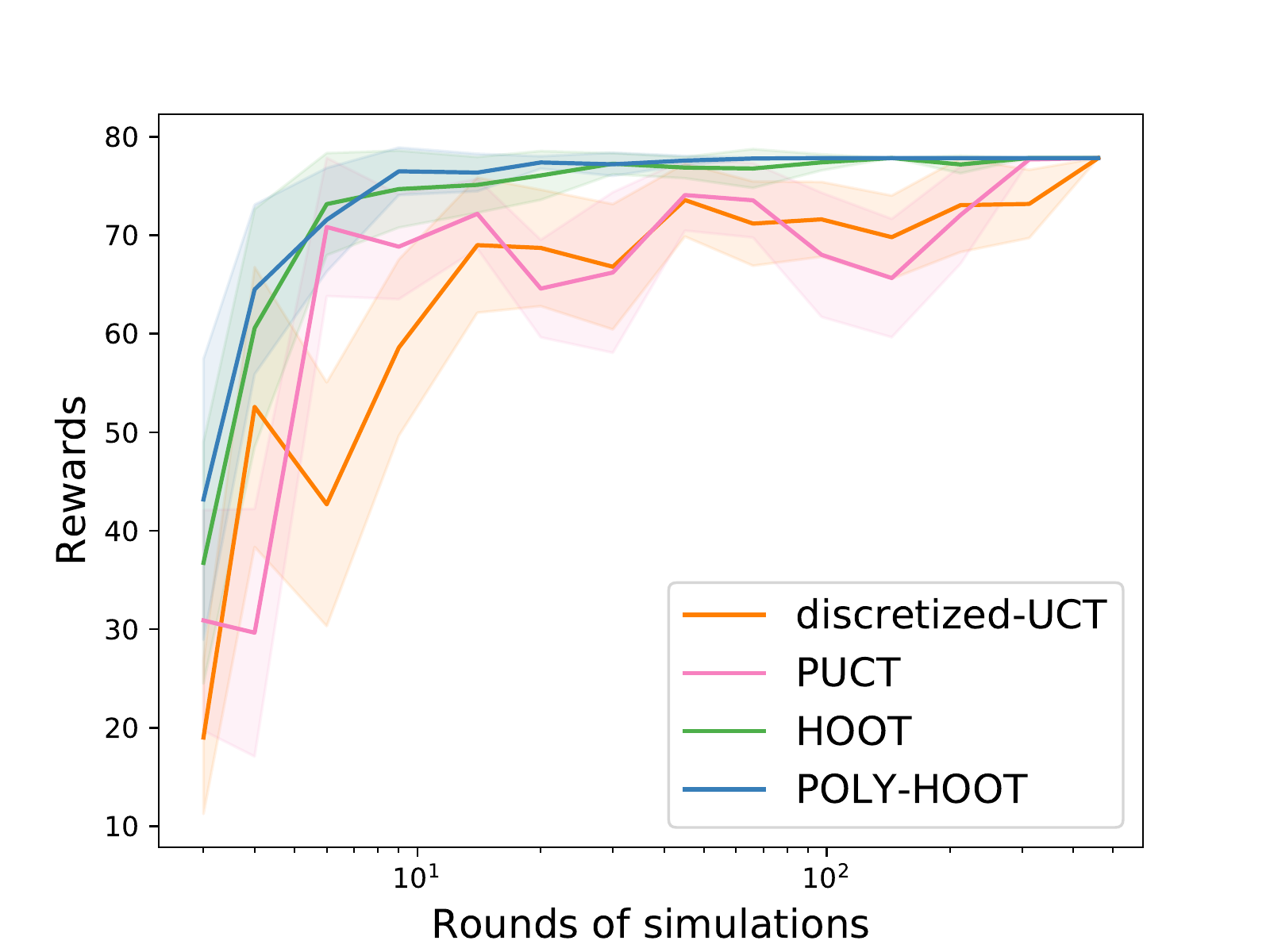}}
	\hspace{-0.55cm}\subfigure[]{\includegraphics[width=0.36\textwidth]{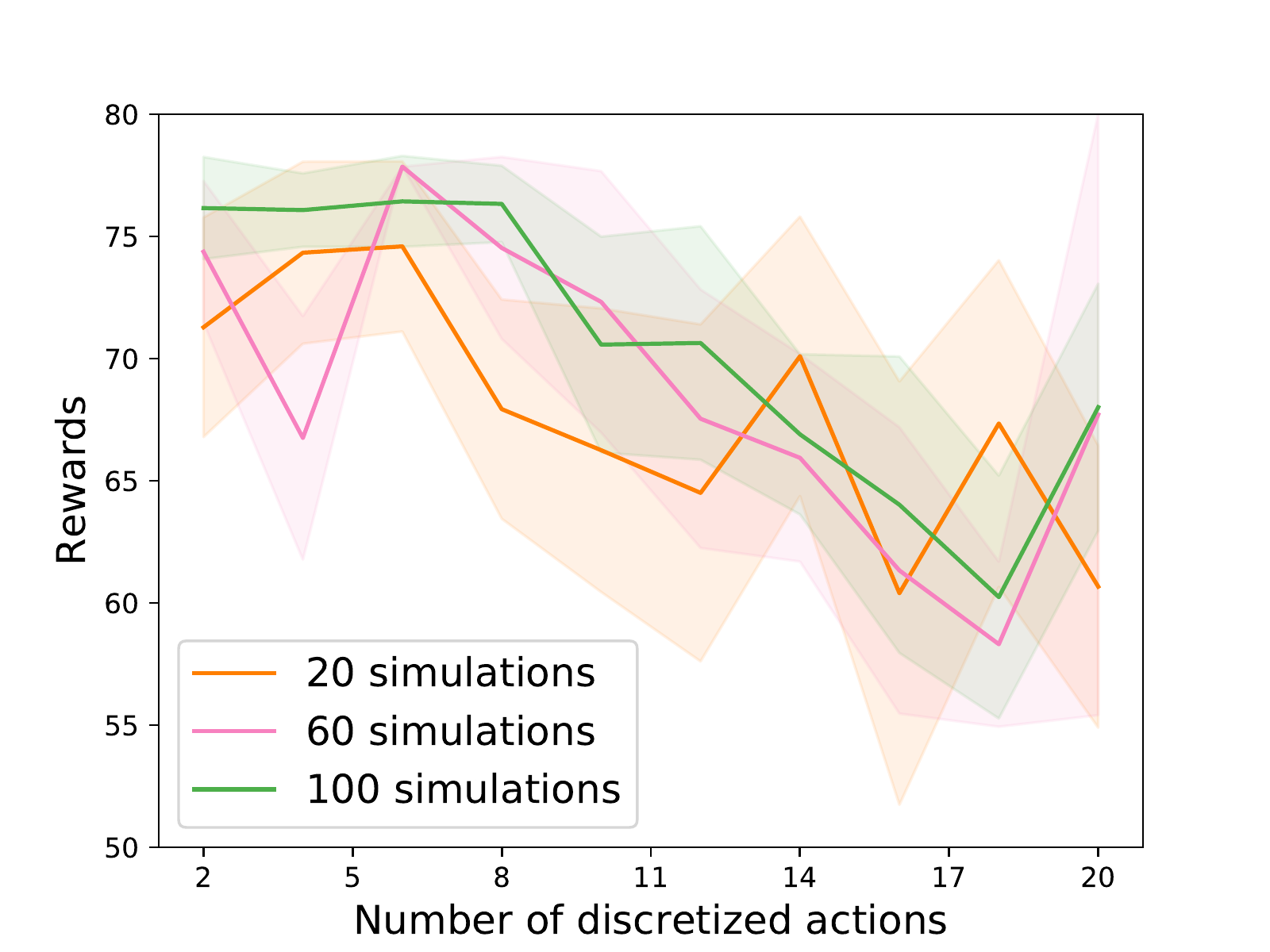}}
	\caption{Figures (a) and (b) show the rewards of the four algorithms with respect to the rounds of simulations per MCTS step on CartPole and CartPole-IG, respectively. The horizontal axes are in logarithmic scales. The shaded areas denote the standard deviations. Figure (c) shows the reward of discretized-UCT with respect to the action discretization level on CartPole-IG. } 
	\label{fig:1}
\end{figure*}

In addition to the evaluation results presented in the main text, we have also tested how the number of simulation rounds per planning step influences the rewards of the four algorithms. The number of simulation rounds is proportional to the number of samples used in each step, and hence we can use this experiment to infer the sample complexities of different algorithms. The evaluation results on CartPole and CartPole-IG are shown in Figures~\ref{fig:1} (a) and (b), respectively. As we can see, HOOT and \texttt{POLY-HOOT} require significantly fewer rounds of simulations to achieve the optimal rewards, which suggests that they have better sample complexities than discretized-UCT and PUCT. 
	 
We have also evaluated how the action discretization level influences the performance of discretized-UCT. The evaluation results on CartPole-IG are shown in Figure~\ref{fig:1} (c), where different curves denote different numbers  of simulation rounds per planning step. As we can see, the performance of discretized-UCT does not necessarily improve with finer granularity of actions. We believe the reason is that, given the fixed number of samples used in each step, each discretized action cannot be well estimated and fully exploited when the discretized action space is large. In addition, there exist huge reward fluctuations even if we only slightly modify the action granularity. This suggests that the performance of discretized-UCT is very sensitive to the discretization level, making this hyper-parameter hard to tune. 
These evaluation results can further demonstrate the advantages of partitioning the action space adaptively in HOOT and \texttt{POLY-HOOT}.